\pdfoutput=1
\def\year{2020}\relax
\documentclass[letterpaper]{article} 
\usepackage{aaai20}  
\usepackage{times}  
\usepackage{helvet} 
\usepackage{courier}  
\usepackage[hyphens]{url}  
\usepackage{graphicx} 
\usepackage{graphicx}  
\usepackage[utf8]{inputenc}
\usepackage{amsmath} 
\usepackage{amssymb}
\usepackage{amsthm}
\usepackage{xcolor}
\usepackage{physics}
\usepackage{float}
\usepackage{amsmath,amsfonts,amssymb,amsthm}
\usepackage{dsfont}
 
\newcommand{\dreal}{\delta_\mu^{\pi,p}(s,a)} 
\newcommand{\ddreal}{\delta_{s,a}^{\pi,p}(s',a')}
\newcommand{\ddfake}{\delta_{s,a}^{\pi,\widehat{p}}(s',a')}
 
\newcommand{\qreal}{Q^{\pi,p}(s,a)}

\newcommand{\ints}{\int_\mathcal{S}}
\newcommand{\inta}{\int_\mathcal{A}}
\newcommand{\intt}{\int_\mathcal{T}}
\newcommand{\qfake}{Q^{\pi,\widehat{p}}(s,a)}
\renewcommand{\d}{\mathrm{d}}
\DeclareMathOperator*{\argmax}{arg\,max}
\DeclareMathOperator*{\argmin}{arg\,min}
\usepackage{thmtools}
\usepackage{thm-restate}
\newtheorem{assumption}{Assumption}

\newtheorem{theorem}{Theorem}[section]
\newtheorem{corollary}{Corollary}[theorem]
\newtheorem{proposition}{Proposition}[section]
\newtheorem{lemma}[theorem]{Lemma}
\newtheorem{definition}[theorem]{Definition}
\usepackage{algorithm}
\usepackage{algpseudocode}
\usepackage{subcaption}
\usepackage{booktabs}
\usepackage{lipsum}
\usepackage{hhline}
\usepackage{mathtools}
\usepackage{wrapfig}
\usepackage{array}
\newcolumntype{P}[1]{>{\centering\arraybackslash}p{#1}}
\newcolumntype{M}[1]{>{\centering\arraybackslash}m{#1}}
\algnewcommand{\LineComment}[1]{\State \(\triangleright\) #1}

\newcommand{\mathbr}[1]{\bm{\mathbf{#1}}}

\usepackage{bm}
\usepackage{amssymb}
\newcommand{\vtheta}{{\bm{\theta}}}

\usepackage{xspace}
\DeclareRobustCommand{\eg}{e.g.,\@\xspace}                         
\DeclareRobustCommand{\ie}{i.e.,\@\xspace}                         
\DeclareRobustCommand{\wrt}{w.r.t.\@\xspace}

\newcommand{\E}{\mathop{\mathbb{E}}}

\usepackage{etoolbox}
\AtBeginEnvironment{proof}{\small}
\usepackage{varwidth}

\newcommand{\citet}[1]{\citeauthor{#1} \shortcite{#1}}
\newcommand{\citep}{\cite}
\newcommand{\citealp}[1]{\citeauthor{#1} \citeyear{#1}}

\allowdisplaybreaks[4]

\title{Gradient-Aware Model-based Policy Search}

\author{Pierluca D'Oro\thanks{Equal contribution.}, Alberto Maria Metelli\footnotemark[1] \\
{\bf \Large Andrea Tirinzoni, Matteo Papini, Marcello Restelli}\\
Dipartimento di Elettronica, Informazione e Bioingegneria, Politecnico di Milano\\
		Piazza Leonardo da Vinci, 32, 20133, Milano, Italy\\
		pierluca.doro@mail.polimi.it, \{albertomaria.metelli, andrea.tirinzoni, matteo.papini, marcello.restelli\}@polimi.it
}

\begin{document}

\maketitle

\begin{abstract}
Traditional model-based reinforcement learning approaches learn a model of the environment dynamics without explicitly considering how it will be used by the agent. In the presence of misspecified model classes, this can lead to poor estimates, as some relevant available information is ignored. 
In this paper, we introduce a novel model-based policy search approach that exploits the knowledge of the current agent policy to learn  an approximate transition model, focusing on the portions of the environment that are most relevant for policy improvement.
We leverage a weighting scheme, derived from the minimization of the error on the model-based policy gradient estimator, in order to define a suitable objective function that is optimized for learning the approximate transition model. Then, we integrate this procedure into a batch policy improvement algorithm, named Gradient-Aware Model-based Policy Search (GAMPS), which iteratively learns a transition model and uses it, together with the collected trajectories, to compute the new policy parameters. Finally, we empirically validate GAMPS on benchmark domains analyzing and discussing its properties.
\end{abstract}

\section{Introduction}
Model-Based Reinforcement Learning (MBRL,~\citealp{sutton2018reinforcement}; \citealp{nguyen2011model}) approaches use the interaction data collected in the environment to estimate its dynamics, with the main goal of improving the sample efficiency of Reinforcement Learning (RL,~\citealp{sutton2018reinforcement}) algorithms.
However, modeling the dynamics of the environment in a thorough way can be extremely complex and, thus, require the use of very powerful model classes and considerable amounts of data, betraying the original goal of MBRL.
Fortunately, in many interesting application domains (\eg robotics), perfectly capturing the dynamics across the whole state-action space is not necessary for a model to be effectively used by a learning agent~\cite{abbeel2006using,nguyen2009model,levine2014learning}. 
Indeed, a wiser approach consists in using simpler model classes, whose estimation requires few interactions with the environment, and focus their limited capacity on the most relevant parts of the environment. 
These parts could present a local dynamics that is inherently simpler than the global one, or at least easier to model using prior knowledge. 

The vast majority of MBRL methods employs a maximum-likelihood estimation process for learning the model~\cite{deisenroth2013survey}.
Nonetheless, the relative importance of the different aspects of the dynamics greatly depends on the underlying decision problem, on the control approach, and, importantly, on the policy played by the agent.
Recent work~\cite{farahmand2017value,farahmand2018iterative} shows that, in the context of value-based reinforcement learning methods, it is possible to derive a \emph{decision-aware} loss function for model learning that compares favorably against maximum likelihood. However, there exists no equivalent for policy-based methods, that are often preferred in the case of continuous observation/action spaces. Moreover, previous work fails at incorporating the influence of the current agent's behavior for evaluating the relative importance of the different aspects of the world dynamics.
For instance, suppose a learning agent acts deterministically in a certain region of the environment, possibly thanks to some prior knowledge, and has no interest in changing its behavior in that area; or that some regions of the state space are extremely unlikely to be reached by an agent following the current policy. There would be no benefit in approximating the corresponding aspects of the dynamics since that knowledge cannot contribute to the agent's learning process. 
Therefore, with a limited model expressiveness, an approach for model learning that explicitly accounts for the current policy and for how it will be improved can outperform traditional maximum likelihood estimation.

In this paper, motivated by these observations, we propose a model-based policy search~\cite{deisenroth2013survey,sutton2000policy} method that leverages awareness of the current agent's policy in the estimation of a forward model, used to perform policy optimization. Unlike existing approaches, which typically ignore all the knowledge available on the running policy during model estimation, we incorporate it into a weighting scheme for the objective function used in model learning. 
We choose to focus our discussion on the batch setting~\cite{lange2012batch}, due to its particular real-world importance. Nonetheless, extensions to the interactive scenario can be easily derived.
The contributions of this paper are theoretical, algorithmic and experimental. After having introduced our notation and the required mathematical preliminaries (Section~\ref{sec:background}), we formalize the concept of \textit{Model-Value-based Gradient} (MVG), an approximation of the policy gradient that combines real trajectories along with a value function derived from an estimated model (Section~\ref{sec:mvg}). MVG allows finding a compromise between the large variance of a Monte Carlo gradient estimate and the bias of a full model-based estimator. Contextually, we present a bound on how the bias of the MVG is related to the choice of an estimated transition model. In Section \ref{sec:gamps}, we derive from this bound an optimization problem to be solved, using samples, to obtain a gradient-aware forward model. Then, we integrate it into a batch policy optimization algorithm, named \textit{Gradient-Aware Model-based Policy Search} (GAMPS), that iteratively uses samples to learn the approximate forward model and to estimate the gradient, used to perform the policy improvement step. After that, we present a finite-sample analysis for the single step of GAMPS (Section~\ref{sec:theoreticalAnalysis}), that highlights the advantages of our approach when considering simple model classes. Finally, after reviewing related work in model-based policy search and decision-aware MBRL areas (Section~\ref{sec:related_works}), we empirically validate  GAMPS against model-based and model-free baselines, and discuss its peculiar features (Section \ref{sec:experiments}). The proofs of all the results presented in the paper are reported in Appendix~\ref{apx:proofs}.

\section{Preliminaries} \label{sec:background}
A discrete-time Markov Decision Process (MDP,~\citealp{puterman2014markov}) is described by a tuple $\mathcal{M} = (\mathcal{S}, \mathcal{A}, r, p, \mu, \gamma)$, where $\mathcal{S}$ is the space of possible states, $\mathcal{A}$ is the space of possible actions, $r(s,a)$ is the reward received by executing action $a$ in state $s$, $p(\cdot|s,a)$ is the transition model that provides the distribution of the next state when performing action $a$ in state $s$, $\mu$ is the distribution of the initial state and $\gamma \in [0,1)$ is a discount factor. When needed, we assume that $r$ is known, as common in domains where MBRL is employed (\eg robotic learning~\cite{deisenroth2011pilco}), and that rewards are uniformly bounded by $|r(s,a)| \le R_{\max} < +\infty$. The behavior of an agent is described by a policy $\pi(\cdot|s)$ that provides the distribution over the action space for every state $s$. Given a state-action pair $(s,a)$ we define the action-value function~\cite{sutton2018reinforcement}, or Q-function, as $Q^{\pi,p}(s,a) = r(s,a) + \gamma \ints p(s'|s,a) \inta \pi(a'|s') Q^{\pi,p}(s',a') \d s' \d a'$ and the state-value function, or V-function, as $V^{\pi,p}(s) = \E_{a \sim \pi(\cdot|s)} [Q^{\pi,p}(s,a)]$, where we made explicit the dependence on the policy $\pi$ and on the transition model $p$. The goal of the agent is to find an optimal policy $\pi^*$, \ie a policy that maximizes the \emph{expected return}: $J^{\pi,p} = \mathbb{E}_{s_0 \sim \mu} \left[ V^{\pi,p}(s_0) \right]$.

We consider a batch setting~\cite{lange2012batch}, in which the learning is performed on a previously collected dataset $\mathcal{D}= \left\{ \tau^{i} \right\}_{i=1}^N = \left\{ \left(s^{i}_0, a^{i}_0, s^{i}_1, a^{i}_1, ...,  s^{i}_{T_i-1}, a^{i}_{T_i-1}, s^{i}_{T_i} \right)\right\}_{i=1}^N $ of $N$ independent trajectories $\tau^{i}$, each composed of $T_i$ transitions, and further interactions with the environment are not allowed.
The experience is generated by an agent that interacts with the environment, following a \emph{known} behavioral policy $\pi_b$. 
We are interested in learning a parameterized policy $\pi_\vtheta$ (for which we usually omit the parameter subscript in the notation) that belongs to a parametric space of stochastic differentiable policies $\Pi_{\Theta} = \{ \pi_\vtheta :\vtheta \in \Theta \subseteq \mathbb{R}^d \}$. In this case, the gradient of the expected return \wrt $\vtheta$ is provided by the \textit{policy gradient theorem} (PGT)~\cite{sutton2000policy,sutton2018reinforcement}:
\begin{equation}
\begin{aligned}
\nabla_\vtheta J(\vtheta) = \frac{1}{1 - \gamma} \ints \inta & \dreal \nabla_\vtheta \log \pi (a | s)  \\ 
& \phantom{} \times \qreal \d s \d a,
\end{aligned}
\end{equation}
where $\dreal$ is the $\gamma$-discounted state-action distribution~\cite{sutton2000policy}, defined as $\dreal = (1-\gamma) \sum_{t=0}^{+\infty} \gamma^t \Pr(s_t=s,a_t=a | \mathcal{M},\pi)$. We call $\nabla_\vtheta \log \pi (a | s)$ the \textit{score} of the policy $\pi$ when executing action $a$ in state $s$.
Furthermore, we denote with $\delta_{s',a'}^{\pi,p}(s,a)$ the state-action distribution under policy $\pi$ and model $p$ when the environment is deterministically initialized by executing action $a'$ in state $s'$ and with $\zeta_\mu^{\pi,p}(\tau)$ the probability density function of a trajectory $\tau$. In batch policy optimization, the policy gradient is typically computed for a policy $\pi$ that is different from the policy $\pi_b$ having generated the data (\textit{off-policy} estimation~\cite{precup2000elegibility}). To correct for the distribution mismatch, we employ \textit{importance sampling}~\cite{kahn1953methods,mcbook}, re-weighing the transitions based on the probability of being observed under the policy $\pi$. Namely, we define the importance weight relative to a subtrajectory $\tau_{t':t''}$ of $\tau$, occurring from time $t'$ to $t''$, and to policies $\pi$ and $\pi_b$ as $\rho_{\pi/\pi_b}(\tau_{t':t''}) = \frac{\zeta_\mu^{\pi,p}(\tau_{t':t''})}{\zeta_\mu^{\pi_b,p}(\tau_{t':t''})} = \prod_{t=t'}^{t''} \frac{\pi(a_t|s_t)}{\pi_b(a_t|s_t)}$.

\section{Model-Value-based Gradient} \label{sec:mvg}
The majority of the model-based policy search approaches employ the learned forward model for generating rollouts, which are used to compute an improvement direction $\nabla_\vtheta J(\vtheta)$ either via likelihood-ratio methods or by propagating gradients through the model~\cite{deisenroth2011pilco}. Differently from these methods, we consider an approximation of the gradient, named \textit{Model-Value-based Gradient} (MVG), defined as follows.
\begin{definition}\label{thr:defMVG}
Let $p$ be the transition model of a Markov Decision Process $\mathcal{M}$, $\Pi_\Theta$ a parametric space of stochastic differentiable policies, $\mathcal{P}$ a class of transition models. Given $\pi \in \Pi_\Theta$ and $\widehat{p} \in \mathcal{P}$, the \emph{Model-Value-based Gradient (MVG)} is defined as:
\begin{equation}
    \label{eq:gradient_approximation}
\begin{aligned}    
    \nabla^{\mathrm{MVG}}_\vtheta J(\vtheta) = \frac{1}{1 - \gamma}\ints \inta & \dreal \nabla_\vtheta \log \pi (a | s)  \\ 
    & \times \qfake \d s \d a.
\end{aligned}
\end{equation}
\end{definition}
Thus, the MVG employs experience collected in the real environment $p$, \ie sampling from $\dreal$, and uses the generative power of the estimated transition kernel $\widehat{p}$ in the computation of an approximate state-action value function $Q^{\pi,\widehat{p}}$ only. 
In this way, it is possible to find a compromise between a full model-based estimator, in which the experience is directly generated from $\delta_\mu^{\pi,\widehat{p}}$~\cite{deisenroth2011pilco,deisenroth2013survey}, and a Monte Carlo estimator (\eg GPOMDP~\cite{baxter2001infinite}) in which also the Q-function is computed from experience collected in the real environment.
Therefore, the MVG limits the bias effect of $\widehat{p}$ to the Q-function approximation $Q^{\pi,\widehat{p}}$.\footnote{It is worth noting that when the environment dynamics can be approximated \emph{locally} with a simple model or some prior knowledge on the environment is available, selecting a suitable approximator $\widehat{p}$ for the transition model is easier than choosing an appropriate function approximator for a critic in an actor-critic architecture.} At the same time, it enjoys a smaller variance \wrt a Monte Carlo estimator, especially in an off-policy setting, as the Q-function is no longer estimated from samples but just approximated using $\widehat{p}$. Existing approaches can be interpreted as MVG. For instance, the ones based on model-based value expansion \cite{Feinberg2018ModelBasedVE,buckman2018sample}, that use a fixed-horizon unrolling of an estimated forward model for obtaining a better value function in an actor-critic setting.
%
%

A central question concerning Definition~\ref{thr:defMVG}, is how the choice of $\widehat{p}$ affects the quality of the gradient approximation, \ie how much bias an MVG introduces in the gradient approximation. To this end, we bound the approximation error by the expected KL-divergence between $p$ and $\widehat{p}$.

\begin{restatable}{theorem}{Weighting}
\label{th:weighting}
Let $q \in [1, +\infty]$ and $\widehat{p} \in \mathcal{P}$. Then, the $L^q$-norm of the difference between the policy gradient $\nabla_\vtheta J(\vtheta)$ and the corresponding MVG $\nabla^{\mathrm{MVG}}_\vtheta J(\vtheta)$ can be upper bounded as:
\begin{align*}
    &\left\| \nabla_\vtheta J(\vtheta) - \nabla^{\mathrm{MVG}}_\vtheta J(\vtheta) \right\|_q \leq \frac{\gamma \sqrt{2} Z R_{\max}}{(1-\gamma)^2}  \\
     &\phantom{\| \nabla_\vtheta J(\vtheta) - \nabla} \times \sqrt{\E_{s,a \sim \eta^{\pi,p}_{\mu}} \left[ D_{KL}(p(\cdot|s,a) \| \widehat{p} (\cdot|s,a)) \right]},
\end{align*}
where
\begin{equation*}
	\eta^{\pi,p}_{\mu}(s,a) =  \ints \inta \nu_\mu^{\pi,p}(s',a') \delta_{s',a'}^{\pi,p}(s,a) \d s' \d a'
\end{equation*} 
is a probability distribution over $\mathcal{S \times A}$, with $\nu_\mu^{\pi,p}(s',a') = \frac{1}{Z} \delta_{\mu}^{\pi, p}(s',a') \left\| \nabla_\vtheta \log \pi_\vtheta (a'|s') \right\|_q$ and $Z =  \ints \inta \delta_{\mu}^{\pi, p}(s',a')  \left\| \nabla_\vtheta \log \pi_\vtheta (a'|s')\right\|_q  \d s' \d a'$ both not depending on $\widehat{p}$.\footnote{We need to assume that $Z > 0$ in order for $\eta^{\pi,p}_{\mu}$ to be well-defined. This is not a limitation, as if $Z=0$, then $\nabla_\vtheta J(\vtheta) = \bm{0}$ and there is no need to define $\eta^{\pi,p}_{\mu}$ in this case.}
\end{restatable}
\begin{proof}[Proof Sketch] Since ${Q^{\pi,p}(s,a) = \int \delta^{\pi,p}_{s,a}(s',a') r(s',a') \d s' \d a'}$, we bound the Q-function difference with $\left\|\delta^{\pi,p}_{s,a} - \delta^{\pi,\widehat{p}}_{s,a}\right\|_1$. The latter is upper bounded with $\left\| p(\cdot|s',a') - \widehat{p}(\cdot|s',a') \right\|_1$. The result follows from Pinsker's inequality.
\end{proof}

Similarly to what was noted for other forms of decision-aware MBRL \cite{farahmand2017value}, a looser bound in which the expectation on the KL-divergence is taken under $\delta_{\mu}^{\pi,p}$ can be derived (Appendix~\ref{sec:gradient_unaware}). This motivates the common maximum likelihood approach.
However, our bound is tighter and clearly shows that not all collected transitions have the same relevance when learning a model that is used in estimating the MVG. 
Overall, the most important $(s,a)$ pairs are those that are \textit{likely to be reached from the policy starting from high gradient-magnitude state-action pairs}. 

\section{Gradient-Aware Model-based Policy Search} \label{sec:gamps}
Inspired by Theorem~\ref{th:weighting}, we propose a policy search algorithm that employs an MVG approximation, combining trajectories generated in the real environment together with a model-based approximation of the Q-function obtained with the estimated transition model $\widehat{p}$. The algorithm, \textit{Gradient-Aware Model-based Policy Search} (GAMPS), consists of three steps: learning the model $\widehat{p}$ (Section~\ref{sec:LearningP}), computing the value function $Q^{\pi,\widehat{p}}$ (Section~\ref{sec:ComputingQ}) and updating the policy using the estimated gradient $\widehat{\nabla}_\vtheta J(\vtheta)$ (Section~\ref{sec:ComputingGrad}).

\subsection{Learning the Transition Model}\label{sec:LearningP}
To learn $\widehat{p}$, we aim at minimizing the bound in Theorem~\ref{th:weighting}, over a class of transition models $\mathcal{P}$, using the trajectories $\mathcal{D}$ collected with $\zeta_\mu^{\pi_b,p}$. However, to estimate an expected value computed over $\eta_{\mu}^{\pi,p}$, as in Theorem~\ref{th:weighting}, we face two problems. First, the policy mismatch between the behavioral policy $\pi_b$ used to collect $\mathcal{D}$ and the current agent's policy $\pi$. This can be easily addressed by using importance sampling. Second, given a policy $\pi$ we need to be able to compute the expectations over $\eta_{\mu}^{\pi,p}$ using samples from $\zeta_\mu^{\pi,p}$. In other words, we need to reformulate the expectation over $\eta_{\mu}^{\pi,p}$ in terms of expectation over trajectories. To this end, we provide the following general result.

\begin{restatable}{lemma}{lemmaObjTraj}\label{thr:lemmaObjTraj}
	Let $\pi$ and $\pi_b$ be two policies such that $\pi \ll \pi_b$ ($\pi$ is absolutely continuous \wrt to $\pi_b$). Let $f : \mathcal{S \times A} \rightarrow \mathbb{R}^k$ be an arbitrary function defined over the state-action space. Then, it holds that:
	\begin{align*}
		& \E_{s,a \sim \eta_\mu^{\pi,p}} \left[f(s,a)\right] = \frac{(1-\gamma)^2}{Z} \E_{\tau \sim \zeta^{\pi_b,p}_{\mu}} \Big[ \sum_{t=0}^{+\infty} \gamma^t  \rho_{\pi/\pi_b} (\tau_{0:t}) \\ 
		& \phantom{\E_{s,a \sim \eta_\mu^{\pi,p}} \left[f(s,a)\right] = Z} \times \sum_{l=0}^t \left\| \nabla_{\vtheta} \log \pi (a_l|s_l) \right\|_q  f(s_{t},a_{t}) \Big].
	\end{align*}
\end{restatable}
To specialize Lemma~\ref{thr:lemmaObjTraj} for our specific case, we just set $f(s,a) = D_{KL}(p(\cdot|s,a) \| \widehat{p} (\cdot|s,a))$. Note that $Z$ is independent from $\widehat{p}$ and thus it can be ignored in the minimization procedure. Furthermore, minimizing the KL-divergence is equivalent to maximizing the log-likelihood of the observed transitions. Putting everything together, we obtain the objective:
\begin{equation}\label{eq:objectiveP}
\begin{aligned}
	\widehat{p} &\in \argmax_{\overline{p} \in \mathcal{P}} \frac{1}{N} \sum_{i=1}^N \sum_{t=0}^{T_i - 1} \omega_t^i \log \overline{p} \left(s_{t+1}^i|s_t^i, a_t^i \right), \\ \omega_t^i &= \gamma^t  \rho_{\pi/\pi_b} (\tau_{0:t}^i) \sum_{l=0}^t \left\| \nabla_{\vtheta} \log \pi (a_l^i|s_l^i)  \right\|_q.
\end{aligned}
\end{equation}

The factors contained in the weight $\omega_t^i$ accomplish three goals in weighting the transitions. 
The discount factor $\gamma^t$ encodes that later transitions are exponentially less important in the gradient computation. The importance weight $\rho_{\pi/\pi_b}(\tau^{i}_{0:t})$ is larger for the transitions that are more likely to be generated by the current policy $\pi$. 
This incorporates a key consideration into model learning: since the running policy $\pi$ can be quite different from the policy that generated the data $\pi_b$, typically very explorative~\cite{deisenroth2013survey}, an accurate approximation of the dynamics for the regions that are rarely reached by the current policy is not useful. 
Lastly, the factor $\sum_{l=0}^t \left\| \nabla_{\vtheta} \log \pi (a_l^i|s_l^i)  \right\|_q$ favors those transitions that occur at the end of a subtrajectory $\tau_{0:t}$ with a high cumulative score-magnitude. This score accumulation resembles the expression of some model-free gradient estimators~\cite{baxter2001infinite}. 
Intuitively, the magnitude of the score of a policy is related to its \emph{opportunity to be improved}, \ie the possibility to change the probability of actions. Our gradient-aware weighting scheme encourages a better approximation of the dynamics for states and actions found in trajectories that can potentially lead to the most significant improvements to the policy.

\subsection{Computing the value function}\label{sec:ComputingQ}
The estimated transition model $\widehat{p}$ can be used to compute the action-value function $Q^{\pi,\widehat{p}}$ for any policy $\pi$. This amounts to \textit{evaluating} the current policy using $\widehat{p}$ instead of the actual transition probability kernel $p$. In the case of finite MDPs, the evaluation can be performed either in closed form or in an iterative manner via dynamic programming~\cite{bellman1954theory,sutton2018reinforcement}. For continuous MDPs, $Q^{\pi,\widehat{p}}$ cannot, in general, be represented exactly. A first approach consists of employing a function approximator $\widehat{Q}\in \mathcal{Q}$ and applying approximate dynamic programming~\cite{bertsekas1995dynamic}. However, this method requires a proper choice of a functional space $\mathcal{Q}$ and the definition of the regression targets, which should be derived using the estimated model $\widehat{p}$~\cite{ernst2005tree,riedmiller2005neural}, possibly introducing further bias.
%

We instead encourage the use of $\widehat{p}$ as a generative model for the sole purpose of approximating $Q^{\pi,\widehat{p}}$. Recalling that we will use $\widehat{Q}$ to estimate the policy gradient from the available trajectories, we can just obtain a Monte Carlo approximation of $Q^{\pi,\widehat{p}}$ on the fly, in an unbiased way, averaging the return from a (possibly large) number $M$ of imaginary trajectories obtained from the estimated model $\widehat{p}$:
\begin{equation}
	\widehat{Q}(s,a) = \frac{1}{M} \sum_{j=1}^M \sum_{t=0}^{T_j-1} \gamma^t r(s_{t}^j, a_{t}^j), \quad \tau^j \sim \zeta_{s,a}^{\pi,\widehat{p}}.
\end{equation}
This approach has the advantage of avoiding the harsh choice of an appropriate model class $\mathcal{Q}$ and the definition of the regression targets, while providing an unbiased estimate for the quantity of interest.

\subsection{Estimating the policy gradient}
\label{sec:ComputingGrad}
After computing $Q^{\pi,\widehat{p}}$ (or some approximation $\widehat{Q}$), all the gathered information can be used to improve policy $\pi$. 
As we are using a \textit{model-value-based gradient}, the trajectories we will use have been previously collected in the real environment.
Furthermore, the data have been generated by a possibly different policy $\pi_b$, and, to account for the difference in the distributions, we need importance sampling again. Therefore, by writing the sample version of Equation \eqref{eq:gradient_approximation} we obtain:
\begin{equation}
\label{eq:gradient_estimate}
\begin{aligned}
    \widehat{\nabla}_\vtheta J (\vtheta) = \frac{1}{N} \sum_{i=1}^{N} \sum_{t=0}^{T_i - 1} & \gamma^t \rho_{\pi/\pi_b}(\tau^{i}_{0:t}) \nabla_\vtheta \log \pi (a^{i}_t| s^{i}_t)  \\
    & \times Q^{\pi,\widehat{p}}(s^{i}_t,a^{i}_t).
\end{aligned}
\end{equation}

For performing batch policy optimization, we repeat the three steps presented in this section using the data collected by the behavior policy $\pi_b$. At each iteration, we fit the model with the weights of the current policy, we employ it in the computation of the state-action value function and we then improve the policy with one or more steps of gradient ascent. The overall procedure is summarized in Algorithm \ref{alg:GAMPS}.

\begin{algorithm}[t]
\small
\caption{Gradient-Aware Model-based Policy Search}
\label{alg:GAMPS}
\textbf{Input:} Trajectory dataset $\mathcal{D}$, behavior policy $\pi_b$, initial parameters $\vtheta_0$, step size schedule $(\alpha_k)_{k=0}^{K-1}$
\begin{algorithmic}[1]
\For{$k = 0,1,...,K-1$}
\LineComment{Learn $\widehat{p}$ (Section~\ref{sec:LearningP})}
\State $\omega_{t,k}^{i} \gets \gamma^t \rho_{\pi_{\vtheta_k}/\pi_b}(\tau^{i}_{0:t}) \sum_{l=0}^t  \| \nabla_\vtheta \log \pi_{\vtheta_k} (a^i_l|s^i_l)\|_q$ 
\State $\widehat{p}_{k} \gets \argmax_{\overline{p} \in \mathcal{P}} \frac{1}{N} \sum_{i=1}^{N}  \sum_t \omega^{i}_{t,k} \log \overline{p}(s^{i}_{t+1}|s^{i}_t,a^{i}_t)$ 
\LineComment{Compute $\widehat{Q}$ (Section~\ref{sec:ComputingQ})}
\State Generate $M$ trajectories for each $(s,a)$ using $\widehat{p}_k$
\State $\widehat{Q}_k(s,a) = \frac{1}{M} \sum_{j=1}^M \sum_{t=0}^{T_j-1} \gamma^t r(s_{t}^j, a_{t}^j)$ 
\LineComment{Improve Policy (Section~\ref{sec:ComputingGrad})}
\State \begin{varwidth}[t]{\linewidth} $\widehat{\nabla}_\vtheta J (\vtheta_k) \gets \frac{1}{N} \sum_{i=1}^{N} \sum_{t=0}^{T_i - 1}  \gamma^t \rho_{\pi_{\vtheta_k}/\pi_b}(\tau^{i}_{0:t}) \times$ 
\par \hskip \algorithmicindent \phantom{$ \widehat{\nabla}_\vtheta J (\vtheta_k) \gets \frac{1}{N} \sum_{i}^{N}$} $ \times \nabla_\vtheta \log \pi_{\vtheta_k} (a^{i}_t| s^{i}_t) \widehat{Q}_k(s^{i}_t,a^{i}_t)$ \end{varwidth} 
\State $\vtheta_{k+1} \gets \vtheta_k + \alpha_k \widehat{\nabla}_\vtheta J (\vtheta_k)$
\EndFor
\end{algorithmic}
\end{algorithm}

\section{Theoretical Analysis}\label{sec:theoreticalAnalysis}
In this section, we provide a finite-sample bound for the gradient estimation of Equation~\eqref{eq:gradient_estimate}, assuming to have the exact value of $Q^{\pi,\widehat{p}}$. This corresponds to the analysis of a single iteration of GAMPS.
We first define the following functions. Let $\tau$ be a trajectory, $\pi \in \Pi_\Theta$ and $\overline{p} \in \mathcal{P}$. We define $l^{\pi,\overline{p}}(\tau) = \sum_{t=0}^{+\infty} \omega_t \log \overline{p} \left(s_{t+1}|s_t, a_t \right)$ and $\mathbr{g}^{\pi,\overline{p}}(\tau) = \sum_{t=0}^{+\infty} \gamma^t \rho_{\pi/\pi_b}(\tau_{0:t})  \nabla_\vtheta \log \pi (a_t| s_t) Q^{\pi,\overline{p}}(s_t,a_t)$. To obtain our result, we need the following assumptions.
\begin{assumption}\label{ass:boundedMoment}
	The second moment of $l^{\pi,\overline{p}}$ and $\mathbr{g}^{\pi,\overline{p}}$ are uniformly bounded over $\mathcal{P}$ and $\Pi_\Theta$. In this case, given a dataset $\mathcal{D} = \{\tau^i\}_{i=1}^N$, there exist two constants ${c_1,\,c_2 < +\infty}$ such that for all $\overline{p} \in \mathcal{P}$ and $\pi \in \Pi_\Theta$:
%
	\begin{align*}
	 & \max\bigg\{ \E_{\tau \sim \zeta_{\mu}^{\pi_b,p}} \left[l^{\pi,\overline{p}}(\tau)^2 \right], \frac{1}{N} \sum_{i=1}^N l^{\pi,\overline{p}}(\tau^i)^2  \bigg\} \le c_1^2 \\
 & \max\bigg\{ \bigg\| \E_{\tau \sim \zeta_{\mu}^{\pi_b,p}} \left[\mathbr{g}^{\pi,\overline{p}}(\tau)^2  \right] \bigg\|_{\infty},\\
 & \phantom{\max\bigg\{} \bigg\| \frac{1}{N} \sum_{i=1}^N \mathbr{g}^{\pi,\overline{p}}(\tau^i)^2 \bigg\|_{\infty} \bigg\}  \le R_{\max}^2 c_2^2.
	\end{align*}
\end{assumption}
\begin{assumption}\label{ass:boundedPdim}
	The pseudo-dimension of the hypothesis spaces $\ \left\{ l^{\pi,\overline{p}} : \overline{p} \in \mathcal{P},\, \pi \in \Pi \right\}$ and $ \left\{ \mathbr{g}^{\pi,\overline{p}} : \overline{p} \in \mathcal{P},\, \pi \in \Pi \right\}$ are bounded by $v < +\infty$.
\end{assumption}
Assumption~\ref{ass:boundedMoment} is requiring that the overall effect of the importance weight $\rho_{\pi/\pi_b}$, the score $\nabla_{\vtheta} \log \pi$ and the approximating transition model $\overline{p}$ preserves the finiteness of the second moment. Clearly, a sufficient (albeit often unrealistic) condition is requiring all these quantities to be uniformly bounded. Assumption~\ref{ass:boundedPdim} is necessary to state learning theory guarantees. We are now ready to present the main result, which employs the learning theory tools of~\citet{cortes2013relative}.

\begin{restatable}{theorem}{mainTheorem}
	Let $q \in [1,+\infty]$, $d$ be the dimensionality of $\Theta$ and $\widehat{p} \in \mathcal{P}$ be the maximizer of the objective function in Equation~\eqref{eq:objectiveP}, obtained with $N>0$ independent trajectories $\{\tau^i\}_{i=1}^N$. Under Assumption~\ref{ass:boundedMoment} and~\ref{ass:boundedPdim}, for any $\delta \in (0,1)$, with probability at least $1-4\delta$ it holds that: 
\begin{align*}
		&\left\| {\widehat{\nabla}}_{\vtheta} J(\vtheta) - \nabla_{\vtheta} J(\vtheta) \right\|_q \le \underbracket{2 R_{\max} \left( d^{\frac{1}{q}} c_2 \epsilon + \frac{\gamma \sqrt{2 Z c_1 \epsilon}}{1 - \gamma} \right)}_{\text{estimation error}}\\
		& +  \underbracket{\frac{\gamma \sqrt{2} Z R_{\max}}{(1-\gamma)^2} \inf_{\overline{p} \in \mathcal{P}}  \sqrt{\E_{s,a \sim \eta^{\pi,p}_{\mu}} \left[ D_{KL}(p(\cdot|s,a) \| \overline{p} (\cdot|s,a)) \right]}}_{\text{approximation error}} ,
\end{align*}
	where $\epsilon = \sqrt{\frac{v \log \frac{2eN}{v} + \log \frac{8(d+1)}{\delta}}{N}} \Gamma \left(\sqrt{\frac{v \log \frac{2eN}{v} + \log \frac{8(d+1)}{\delta}}{N}} \right)$ and $\Gamma(\xi) \coloneqq \frac{1}{2} + \sqrt{ 1 + \frac{1}{2} \log \frac{1}{\xi} } $.
\end{restatable}
The theorem justifies the intuition behind the gradient estimation based on MVG.
A model $\overline{p}$ is good when it achieves a reasonable trade-off between the errors in approximation and estimation.\footnote{It is worth noting that the estimation error is $\widetilde{\mathcal{O}}(N^{-\frac{1}{4}})$.} In the case of scarce data (\ie small $N$), it is convenient to choose a low-capacity model class $\mathcal{P}$ in order to reduce the error-enlarging effect of the pseudo-dimension $v$. 
However, this carries the risk of being unable to approximate the original model. 
Nonetheless, the approximation error depends on an expected value under $\eta_\mu^{\pi,p}$. Even a model class that would be highly misspecified \wrt an expectation computed under the state-action distribution $\delta_\mu^{\pi,p}$ can, perhaps surprisingly, lead to an accurate gradient estimation using our approach.

\section{Related Works} \label{sec:related_works}

We now revise prior work in MBRL, focusing on policy search methods and those that include some level of awareness of the underlying control problem into model learning. 

{\bf Policy Search with MBRL~~} The standard approach consists in using a maximum likelihood estimation of the environment dynamics to perform simulations (or \textit{imaginary rollouts}) through which a policy can be improved without further or with limited interactions with the environment~\cite{deisenroth2013survey}.
This approach has taken different forms, with the use of tabular models~\cite{wang2003model}, least-squares density estimation techniques~\cite{tangkaratt2014model} or, more recently, combinations of variational generative models and recurrent neural networks employed in world models based on mixture density networks~\cite{ha_recurrent_2018}. Several methods incorporate the model uncertainty into policy updates, by using Gaussian processes and moment matching approximations~\cite{deisenroth2011pilco}, Bayesian neural networks~\cite{McAllister2016ImprovingPW} or ensembles of forward models~\cite{Chua2018DeepRL,kurutach2018modelensemble,janner2019trust,buckman2018sample}. MBRL works that are particularly related to GAMPS are those employing estimated forward models that are accurate only locally \cite{abbeel2006using,nguyen2009model,levine2014learning}, or using a \textit{model-value based gradient} formulation~\cite{abbeel2006using,Feinberg2018ModelBasedVE,buckman2018sample,heess2015learning} as described in Section~\ref{sec:mvg}.


{\bf Decision-aware MBRL~~} The observation that, under misspecified model classes, the dynamics of the environment must be captured foreseeing the final task to be performed led to the development of decision-aware approaches for model learning~\cite{farahmand2017value}.
While one of the first examples was a financial application~\cite{Bengio1997UsingAF}, the idea was introduced into MBRL~\cite{Joseph2013ReinforcementLW,Bansal2017GoaldrivenDL} and the related adaptive optimal control literature~\cite{Piroddi2003simulation} by using actual evaluations of a control policy in the environment as a performance index for model learning. 
More similarly to our approach, but in the context of value-based methods, a theoretical framework called \textit{value-aware} MBRL~\cite{farahmand2017value} was proposed, in which the model is estimated by minimizing the expected error on the Bellman operator, explicitly considering its actual use in the control algorithm. 
Starting from this, further theoretical considerations and approaches have been proposed~\cite{farahmand2018iterative,asadi2018equivalence}. 
Awareness of the final task to be performed has been also incorporated into stochastic dynamic programming~\cite{Donti2017TaskbasedEM,amos2018differentiable} and, albeit implicitly, into neural network-based works~\cite{Oh2017ValuePN,Silver2017ThePE,luo2018algorithmic}, in which value functions and models consistent with each other are learned.

\section{Experiments}\label{sec:experiments}
We now present an experimental evaluation of GAMPS, whose objective is two-fold: assessing the effect of our weighting scheme for model learning and comparing the performance in batch policy optimization of our algorithm against model-based and model-free policy search baselines.

\begin{figure*}[!h]
	\begin{subfigure}{.15\textwidth}
  		\centering
  		\includegraphics[width=\linewidth]{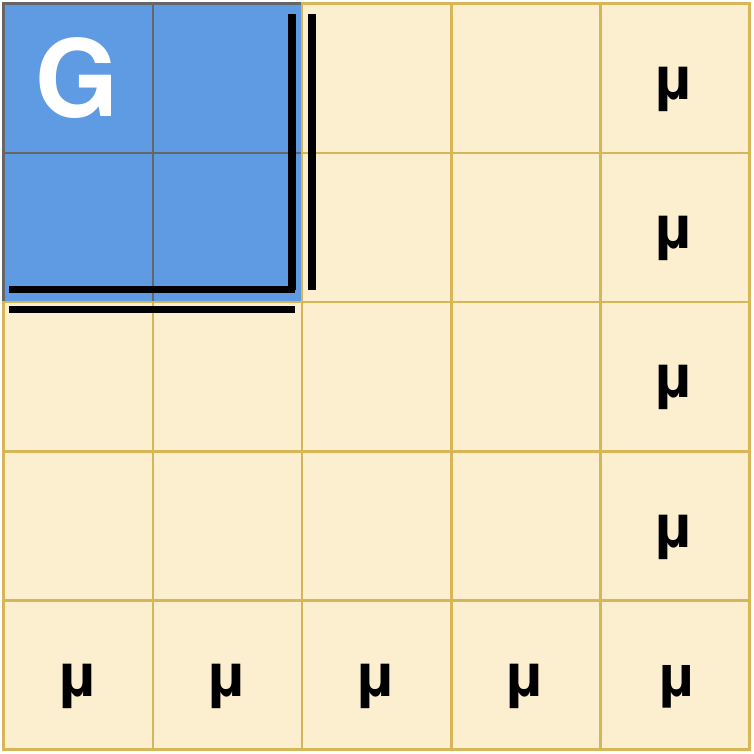}		
  		\caption{Gridworld}
  		\label{subfig:gridworld}
	\end{subfigure}
	\hspace{5pt}
	\begin{subfigure}{.39\textwidth}
  		\centering
  		\includegraphics[width=\linewidth]{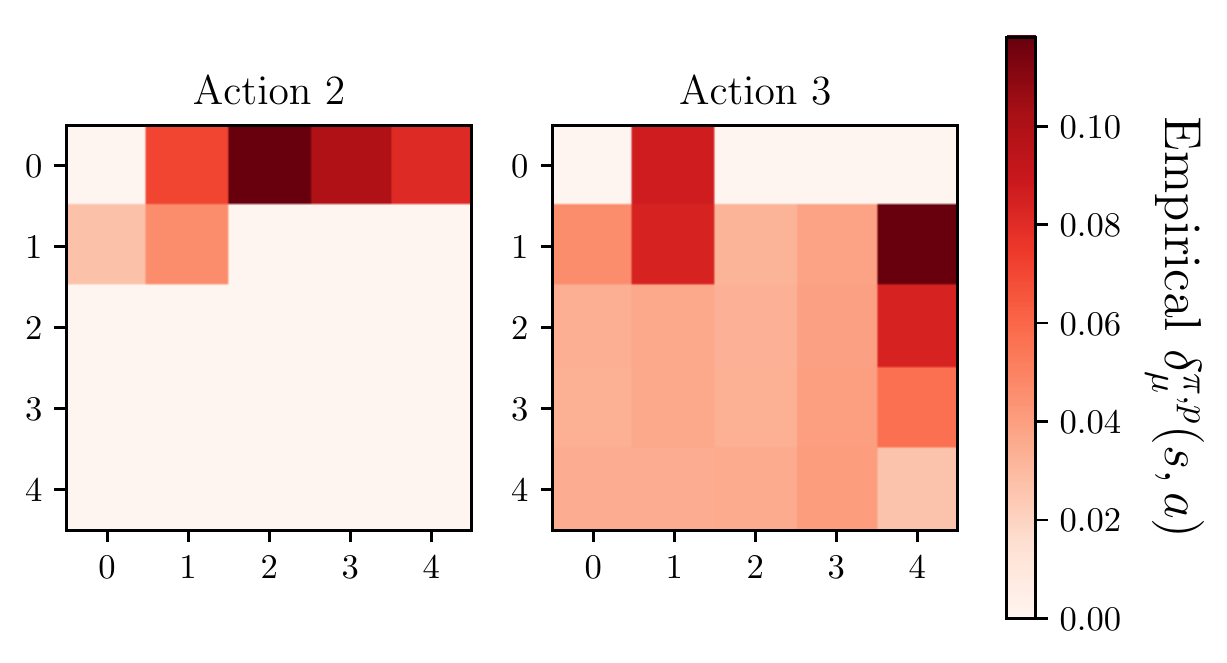} 	
  		\caption{State-action distribution weights}
  		\label{subfig:ml_weights}
	\end{subfigure}
	\begin{subfigure}{.39\textwidth}
  		\centering
  		\includegraphics[width=\linewidth]{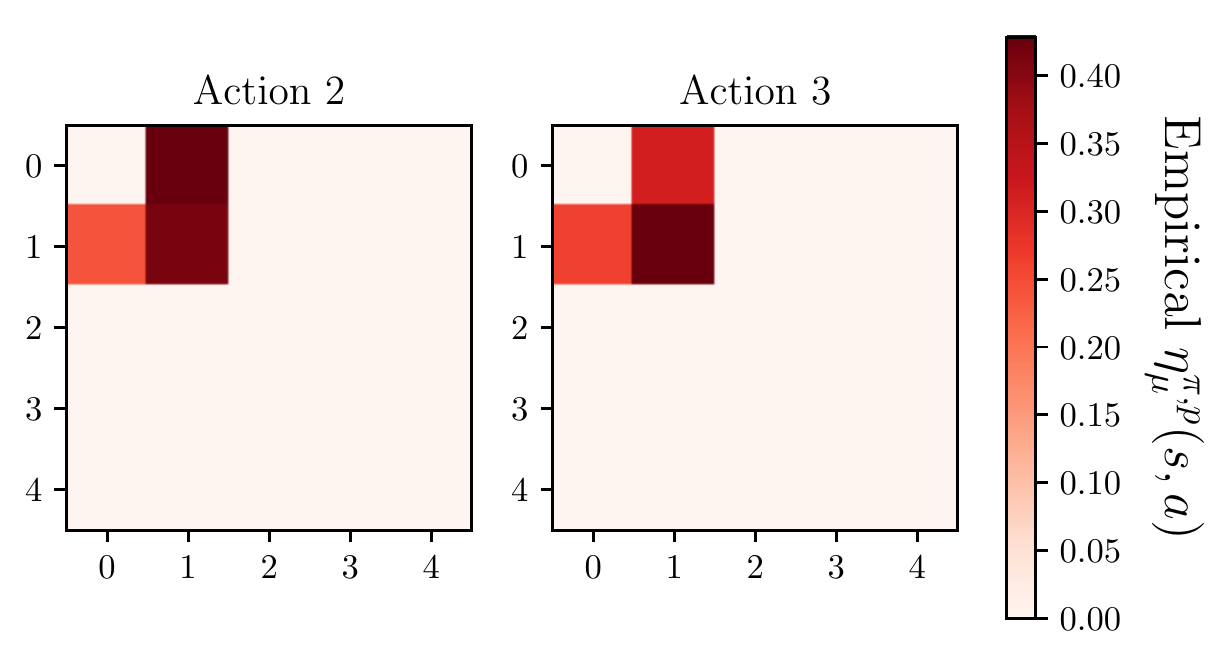}  		
  		\caption{Gradient-aware weights}
  		\label{subfig:ga_weights}
	\end{subfigure}
\caption{(\subref{subfig:gridworld}): Gridworld representation. The goal state is \textbf{G} and the possible initial states are $\mu$. The two areas with different dynamics are represented with different colors. (\subref{subfig:ml_weights}) and (\subref{subfig:ga_weights}): Normalized values of the empirical state-action distribution and the weighting factor for GAMPS. Each grid represents every state of the environment for the two most representative actions.}
\label{fig:weights}
\end{figure*}

\subsection{Two-areas Gridworld}
This experiment is meant to show how decision-awareness can be an effective tool to improve the accuracy of policy gradient estimates when using a forward model. The environment is a $5 \times 5$ gridworld, divided into two areas (lower and upper) with different dynamics: the effect of a movement action of the agent is reversed in one area \wrt the other. Once the agent gets to the lower area, it is not possible to go back in the upper one (Figure~\ref{subfig:gridworld}). We collect experience with a linear policy $\pi_b$ that is deterministic on the lower area and randomly initialized in the upper area, which is also used as initial policy for learning.

The first goal of this experiment is to show that, with the use of gradient-awareness, even an extremely simple model class can be sufficiently expressive to provide an accurate estimate of the policy gradient. Hence, we use a forward model which, given the sole action executed by the agent (\ie without knowledge of the current state), predicts the effect that it will cause on the position of the agent (\ie up, down, left, right, stay). 
This model class cannot perfectly represent the whole environment dynamics at the same time, as it changes between the two areas. However, given the nature of policy $\pi$, this is not necessary, since only the modeling of the upper area, which is indeed representable with our model, would be enough to perfectly improve the policy. Nonetheless, this useful information has no way of being captured using the usual maximum likelihood procedure, which, during model learning, weighs the transitions just upon visitation, regardless of the policy. 
To experimentally assess how our approach addresses this intuitive point, we generate 1000 trajectories running $\pi_b$ in the environment, and we first compare the maximum likelihood and the gradient-aware weighting factors, $\dreal$ and $\eta_{\mu}^{\pi,p} (s,a)$. 
The results (Figure~\ref{subfig:ml_weights} and \ref{subfig:ga_weights}) show that our method is able, in a totally automatic way, to avoid assigning importance to the transitions in which the policy cannot be improved. 

We further investigate the performance of GAMPS compared to batch learning with the maximum likelihood transition model (ML) and two classical model-free learning algorithms REINFORCE~\cite{williams1992simple} and PGT~\cite{sutton2000policy}. To adapt the latter two to the batch setting, we employ importance sampling in the same way as described in Equation \eqref{eq:gradient_estimate}, but estimating the Q-function using the same trajectories (and importance sampling as well). The results obtained by collecting different numbers of trajectories and evaluating on the environment are shown in Figure \ref{fig:gridworld_results}. When the data are too scarce, all compared algorithms struggle in converging towards good policies, experiencing high variance.\footnote{In batch learning, performance degradation when the current policy $\pi$ becomes too dissimilar from the behavioral policy $\pi_b$ is natural due to the variance of the importance weights. To avoid this effect, a stopping condition connected to the effective sample size~\cite{mcbook} can be employed.} It is worth noting that, with any amount of data we tested, the GAMPS learning curve is consistently above the others, showing superior performance even when considering the best iteration of all algorithms.


\begin{figure*}[t]
\includegraphics[scale=1]{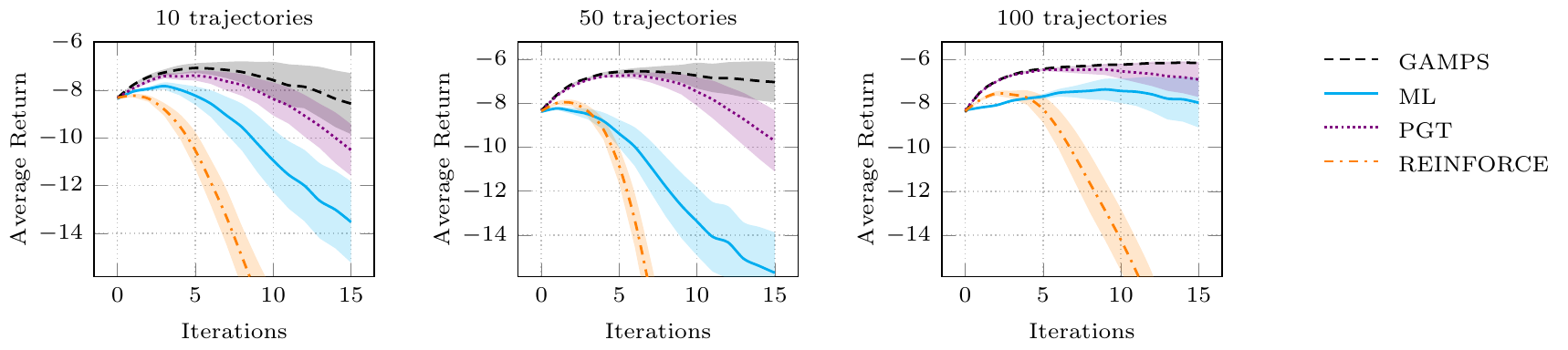}
\caption{Average return on the gridworld. ML employs maximum-likelihood model estimation (20 runs, mean $\pm$ std).}
\label{fig:gridworld_results}
\end{figure*}


\begin{figure}[t]
\centering
	\includegraphics[scale=1]{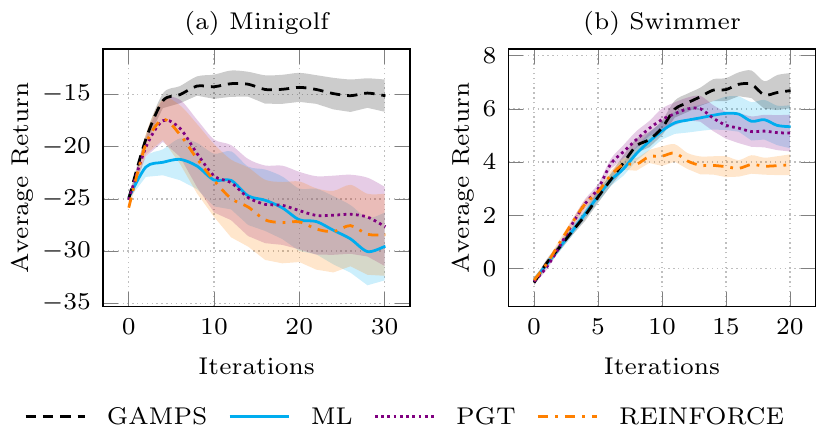}
\caption{Average return in the minigolf domain using 50 trajectories and in the swimmer environment using 100 trajectories (10 runs, mean $\pm$ std).}
\label{fig:minigolf_results}
\end{figure}


\subsection{Continuous Control}
To show that our algorithm is able to perform well also on continuous domains, we test its performance on a simulated Minigolf environment~\cite{lazaric2008reinforcement,tirinzoni2019transfer} and the 3-link Swimmer robot control benchmark based on Mujoco~\shortcite{todorov2012mujoco}.

In the minigolf game, the agent hits a ball using a flat-faced golf club (the putter) with the goal of reaching a hole in the minimum number of strokes. 
Given only the distance to the hole, the agent chooses, at each step, the angular velocity of the putter which determines the next position of the ball. 
The episode terminates when the ball enters the hole, with reward $0$, or when the agent overshoots, with reward $-100$. In all other cases, the reward is $-1$ and the agent can try other hits. We further suppose that the minigolf course is divided into two areas, one twice larger than the other, with different terrains: the first, nearest to the hole and biggest, has the friction of a standard track; the second has very high friction, comparable to the one of an area with sand. 
We use Gaussian policies that are linear on six radial basis function features. The model predicts the difference from the previous state by sampling from a Gaussian distribution with linearly parameterized mean and standard deviation. 

For the Mujoco swimmer the goal of the agent is to swim forward in a fluid. The policy is linear in the state features and the forward model is a 2-layer neural networks with 32 hidden neurons and tanh activation.

We evaluate GAMPS against the same baselines employed for the previous experiment. We collect a dataset of 50 and 100 trajectories for minigolf and swimmer respectively, using an explorative policy, and then run the algorithms for 30 and 20 iterations.
Results in terms of average return are reported in Figure~\ref{fig:minigolf_results}, showing that GAMPS outperforms all the other algorithms both in terms of terminal and maximum performance. 
In the minigolf domain, it is less susceptible to overfitting compared to the baselines, and it obtains a final policy able to reach the hole most of the time.
In the swimmer environment, where powerful models are used, GAMPS still shows superior performance, validating the insight provided by Theorem~\ref{th:weighting} even in this setting.

\section{Discussion and Conclusions}
In this paper, we presented GAMPS, a batch gradient-aware model-based policy search algorithm. GAMPS leverages the knowledge about the policy that is being optimized for learning the transition model, by giving more importance to the aspects of the dynamics that are more relevant for improving its performance. 
We derived GAMPS from the minimization of the bias of the model-value-based gradient, an approximation for the policy gradient that mixes trajectories collected in the real environment together with a value function computed with the estimated model. 
Our theoretical analysis validates the intuition that, when dealing with low capacity models, it is convenient to focus their representation capabilities on the portions of the environment that are most crucial for improving the policy. 
The empirical validation demonstrates that, even when extremely simple model classes are considered, GAMPS is able to outperform the baselines. 

The main limitations of GAMPS are in the need to reuse all the samples in model learning at each iteration and in the compounding errors during rollouts.
Future work could focus on mitigating these issues, as well as on adapting GAMPS to the interactive scenario, by leveraging existing on/off-policy approaches~\cite{metelli2018policy}, and on different ways to exploit gradient-aware model learning.

\section*{Acknowledgments}
This work has been partially supported by the Italian MIUR PRIN 2017 Project ALGADIMAR ``Algorithms, Games, and Digital Market''.
 
{
\fontsize{9pt}{10pt} \selectfont
\bibliography{model_based}
\bibliographystyle{aaai}
}

\onecolumn
\appendix
\section{Proofs and Derivations}\label{apx:proofs}
In this appendix, we report the proofs of the results presented in the main paper, together with some additional results and extended discussion. 

\subsection{Proofs of Section~\ref{sec:mvg}}
The following lemma is used in proving Theorem \ref{th:weighting}.
\begin{lemma}
    \label{daction_to_dstate}
    Considering the state-action distributions $\delta_{\mu}^{\pi, p}$ and $\delta_{\mu}^{\pi, \widehat{p}}$ under policy $\pi$ and models $p$ and $\widehat{p}$, the following upper bound holds:
    \begin{equation}
        \nonumber
        \left\| \delta_{\mu}^{\pi, p} - \delta_{\mu}^{\pi,\widehat{p}} \right\|_1 \leq
     \frac{\gamma}{1 - \gamma} \E_{s,a \sim \delta_{\mu}^{\pi,p}}\left[ \left\| p(\cdot|s,a) - \widehat{p}(\cdot|s,a) \right\|_1 \right].
    \end{equation}
\end{lemma}

\begin{proof}
Recalling that $\delta_{\mu}^{\pi, p}(s,a) = \pi(a|s)d_{\mu}^{\pi,p}(s)$ we can write:
\begin{align*}
 \left\| \delta_{\mu}^{\pi,p} - \delta_{\mu}^{\pi,\widehat{p}} \right\|_1 &= \ints \inta \left|\delta_{\mu}^{\pi,\widehat{p}}(s,a) - \delta_{\mu}^{\pi,\widehat{p}}(s,a)\right| \d s \d a \\
& = \ints \inta \pi(a|s) \left| d_{\mu}^{\pi,p}(s) - d_{\mu}^{\pi,\widehat{p}}(s) \right| \d s \d a \\
& = \ints \left|d_{\mu}^{\pi,p}(s) - d_{\mu}^{\pi,\widehat{p}}(s)\right| \inta \pi(a|s) \d a \d s\\
& = \ints \left|d_{\mu}^{\pi,p}(s) - d_{\mu}^{\pi,\widehat{p}}(s)\right|  \d s  = \left\| d_{\mu}^{\pi,p} - d_{\mu}^{\pi,\widehat{p}} \right\|_1,
\end{align*}

where $d_\mu^{\pi,p}(s) = (1-\gamma) \sum_{t=0}^{+\infty} \gamma^t \Pr(s_t=s | \mathcal{M},\pi)$. In order to bound $\left\| d_{\mu}^{\pi,p} - d_{\mu}^{\pi,\widehat{p}} \right\|_1$, we can use Corollary 3.1 from~\cite{metelli2018configurable}:
\begin{equation*}
	\left\| d_{\mu}^{\pi,p} - d_{\mu}^{\pi,\widehat{p}} \right\|_1 \le \frac{\gamma}{1-\gamma} \E_{s,a \sim \delta_{\mu}^{\pi,p}} \left[ \| p(\cdot|s,a) - \widehat{p}(\cdot|s,a) \|_1 \right].
\end{equation*}
\end{proof}

Now, we can prove Theorem \ref{th:weighting}. 
{\renewcommand\footnote[1]{}\Weighting*}
\begin{proof}
\begin{align}
    \Big\|\nabla_\vtheta & J(\vtheta)  - \nabla^{\text{MVG}}_\vtheta J (\vtheta) \Big\|_q =
    \left\| \frac{1}{1-\gamma}\ints \inta (\dreal (\qreal - \qfake) \nabla_\vtheta \log \pi (a | s) \text{d}s \text{d}a \right\|_q \nonumber \\
    &\leq \frac{1}{1-\gamma} \label{eq:assumption_eq} \ints \inta \dreal \left|\qreal - \qfake \right| \| \nabla_\vtheta \log \pi (a | s) \|_q \text{d}s \text{d}a \\
    &= \label{eq:nu_eq} \frac{Z}{1-\gamma}  \ints \inta  \nu_\mu^{\pi,p} (s,a) \left| \qreal - \qfake \right| \d s \d a \\
    &= \frac{Z}{1-\gamma} \ints \inta \nu_\mu^{\pi,p} (s,a) \left| \ints \inta r(s',a') (\ddreal - \ddfake) \text{d}s'\text{d}a' \right| \d s \d a \label{eq:Qdef}\\
    & \leq \frac{ZR_{\max}}{1-\gamma} \ints \inta \nu_\mu^{\pi,p} (s,a) \left| \ints \inta (\ddreal - \ddfake) \text{d}s'\text{d}a' \right| \d s \d a \label{eq:boundr} \\
    & \leq \frac{ZR_{\max}}{1-\gamma} \ints \inta \nu_\mu^{\pi,p} (s,a)  \left\| \delta_{s,a}^{\pi,p} - \delta_{s,a}^{\pi,\widehat{p}} \right\|_1 \d s \d a \nonumber \\ 
    & \label{eq:dmu_to_model} \leq \frac{ZR_{\max}\gamma}{(1-\gamma)^2} \ints \inta \nu_\mu^{\pi,p} (s,a) \ints \inta \ddreal \left\| p(\cdot|s',a') - \widehat{p} (\cdot|s',a') \right\| \d s'\d a' \d s \d a \\
    & = \frac{ZR_{\max}\gamma}{(1-\gamma)^2} \ints \inta \eta_\mu^{\pi,p}(s',a') \ints  \left| p(s''|s',a') - \widehat{p} (s''|s',a') \right| \d s'' \d s' \d a' \nonumber \\
    & \label{eq:kl_noreward} \leq \frac{ZR_{\max}\gamma}{(1-\gamma)^2} \ints \inta \eta_\mu^{\pi,p}(s,a) \sqrt{2 D_{KL}(p(\cdot|s,a) \| \widehat{p} (\cdot|s,a))} \d s \d a \\
    & \leq \frac{ZR_{\max} \gamma}{(1-\gamma)^2} \sqrt{2 \ints \inta \eta_\mu^{\pi,p} (s,a) D_{KL}(p(\cdot|s,a) \| \widehat{p} (\cdot|s,a)) \d s \d a},
\end{align}
where in Equation \eqref{eq:nu_eq}, we define a new probability distribution $\nu_\mu^{\pi,p}(s,a) = \frac{1}{Z}\dreal \| \nabla_\vtheta \log \pi (a | s) \|_q $ by means of an appropriate normalization constant $Z$, assumed $Z>0$. In Equation~\eqref{eq:Qdef}, we use the definition of Q-function as $Q^{\pi,p}(s,a) = \ints \inta \delta_{s,a}^{\pi,p}(s',a') r(s',a') \d s' \d a'$.
After bounding the reward in Equation~\eqref{eq:boundr}, in Equation \eqref{eq:dmu_to_model} we apply Lemma \ref{daction_to_dstate}.
Then we obtain Equation \eqref{eq:kl_noreward} by employing Pinsker's inequality, defining the overall weighting term $\eta_{\mu}^{\pi,p}(s',a') = \ints \inta \nu_\mu^{\pi,p}(s,a) \ddreal \d s \d a $, and renaming variables for clarity. Last passage follows from Jensen inequality.
\end{proof}

In order to understand how the weighting distribution $\eta^{\pi,p}_{\mu}$ enlarges the relative importance of some transitions with respect to others, we can focus on the auxiliary distribution $\nu_\mu^{\pi,p}$, defined as:
\begin{equation}
 \label{eq:nu_definition}
\nu_\mu^{\pi,p}(s',a')~=~\frac{1}{Z} \left\| \nabla_\vtheta \log \pi (a' | s') \right\|_q \delta_{\mu}^{\pi, p}(s',a'),
\end{equation}
where $Z$, as defined in Theorem~\ref{th:weighting}, is a normalization constant required for $\nu_\mu^{\pi,p}$ to be a well-defined probability distribution. $Z$ can be seen as the \emph{expected score magnitude} in the MDP $\mathcal{M}$ under policy $\pi$. 
The distribution $\nu_\mu^{\pi,p}$ is high for states and actions that are both likely to be visited executing $\pi$ and corresponding to high norm of its score. 
Intuitively, a low magnitude for the score is related to a smaller possibility for policy $\pi$ to be improved.
However, the connection between the score-magnitude for states and actions and the relative importance of those states and actions for minimizing the approximation error caused by the MVG is not direct.
In other words, it is not possible to say that the most important transitions to be learned for a model to be good for an MVG approach are the ones featuring the largest score magnitude and frequently encountered by the policy (\ie $\nu_\mu^{\pi,p}$ is \textit{not} the correct weighting distribution).
Nonetheless, $\nu_\mu^{\pi,p}$ plays an important role in defining the whole weighting distribution $\eta_\mu^{\pi,p}$. In fact, it can be rewritten as:
\begin{equation}
\label{eq:eta_redef}
\eta^{\pi,p}_{\mu}(s,a)=\ints \inta\underbracket{\nu_\mu^{\pi,p}(s',a')}_{\substack{\text{gradient magnitude} \\\ \text{distribution}}} \underbracket{\delta_{s',a'}^{\pi,p}(s,a)}_{\substack{\text{state-action} \\\ \text{reachability}}} \quad \d s' \d a' = \E_{s',a' \sim \nu_\mu^{\pi,p}}\left[ \delta_{s',a'}^{\pi,p}(s,a) \right].
\end{equation}
Under the interpretation suggested by Equation~\ref{eq:eta_redef}, $\eta_\mu^{\pi,p}$ can be seen as the expected \textit{state-action reachability} under the \textit{gradient magnitude distribution}.
$\delta_{s',a'}^{\pi,p}(s,a)$ is the state-action distribution of $(s,a)$ after executing action $a'$ in state $s'$. 
It is equivalent to the state-action distribution $\delta_\mu^{\pi,p}$ in an MDP where 
$\mu(s,a) = \mathds{1} \left(s=s',\, a=a' \right)$.
Each state-action couple $(s',a')$ with high score magnitude that precedes $(s,a)$ brings a contribution to the final weighting factor for $(s,a)$.

\subsection{Gradient-Unaware Model Learning} \label{sec:gradient_unaware}
We now show that maximum-likelihood model estimation is a sound way of estimating the policy gradient when using the MVG, although it is optimizing a looser bound with respect to the one provided by Theorem~\ref{th:weighting}. For proving the following result, we assume the score is bounded by $\| \nabla_\vtheta \log \pi(a|s) \|_q \le K$. 

\begin{proposition}
    Let $q \in [1, +\infty]$ and $\widehat{p} \in \mathcal{P}$. If $\| \nabla_\vtheta \log \pi(a|s) \|_q \le K < +\infty$ for all $s \in \mathcal{S}$ and $s \in \mathcal{A}$, then, the $L^q$-norm of the difference between the policy gradient $\nabla_\vtheta J(\vtheta)$ and the corresponding MVG $\nabla^{\mathrm{MVG}}_\vtheta J(\vtheta)$ can be upper bounded as:
    \begin{equation*}
	    \left\|\nabla_\vtheta J(\vtheta) - \nabla^{\mathrm{MVG}}_\vtheta J(\vtheta) \right\|_q
	    \leq  \frac{ K R_{\max}\sqrt{2} \gamma}{(1-\gamma)^2} \sqrt{\E_{s,a \sim \delta_\mu^{\pi,p}} \left[ D_{KL}(p(\cdot|s,a) \| \widehat{p} (\cdot|s,a)) \right]}.
    \end{equation*}
\end{proposition}
\begin{proof}
\begin{align}
	    \Big\|\nabla_\vtheta & J(\vtheta) - \nabla^{\mathrm{MVG}}_\vtheta J(\vtheta) \Big\|_q  \le \frac{\gamma \sqrt{2} Z R_{\max}}{(1-\gamma)^2} \left({\ints \inta \eta^{\pi,p}_{\mu}(s,a) D_{KL}(p(\cdot|s,a) \| \widehat{p} (\cdot|s,a))  \d s \d a }\right)^{\frac{1}{2}} \notag \\
	    &  = \frac{\gamma \sqrt{2} Z R_{\max}}{(1-\gamma)^2} \bigg(\ints \inta \frac{1}{Z} \ints \inta \|\nabla_\vtheta \log \pi (a' | s')\|_q \delta_{\mu}^{\pi, p}(s',a')  \delta_{s',a'}^{\pi,p}(s,a) \d s' \d a' D_{KL}(p(\cdot|s,a) \| \widehat{p} (\cdot|s,a))  \d s \d a \bigg)^{\frac{1}{2}}\notag \\
	    & \le \frac{\gamma \sqrt{2KZ} R_{\max}}{(1-\gamma)^2} \left({\ints \inta  \ints \inta \delta_{\mu}^{\pi, p}(s',a')  \delta_{s',a'}^{\pi,p}(s,a) \d s' \d a' D_{KL}(p(\cdot|s,a) \| \widehat{p} (\cdot|s,a))  \d s \d a }\right)^{\frac{1}{2}} \notag \\
	    & = \frac{\gamma \sqrt{2KZ} R_{\max}}{(1-\gamma)^2} \left(\ints \inta \delta_{\mu}^{\pi, p}(s,a) D_{KL}(p(\cdot|s,a) \| \widehat{p} (\cdot|s,a))  \d s \d a \right)^{\frac{1}{2}}\label{eq:ddintegral} \\
	    & \le \frac{\gamma \sqrt{2} K R_{\max}}{(1-\gamma)^2} \left(\ints \inta \delta_{\mu}^{\pi, p}(s,a) D_{KL}(p(\cdot|s,a) \| \widehat{p} (\cdot|s,a))  \d s \d a \right)^{\frac{1}{2}}, \label{eq:finalstep}
\end{align}
where we started from Theorem \ref{th:weighting}. Equation~\eqref{eq:ddintegral} follows from the fact that $\int \delta_{\mu}^{\pi, p}(s',a')  \delta_{s',a'}^{\pi,p}(s,a) \d s' \d a' = \delta_{\mu}^{\pi,p}(s,a)$, as we are actually recomposing the state-action distribution that was split at $(s',a')$ and Equation~\eqref{eq:finalstep} is obtained by observing that $Z \le K$.
\end{proof}

Therefore, the bound is looser than the one presented in Theorem \ref{th:weighting}. We can in fact observe that
\begin{align*}
\|\nabla_\vtheta J(\vtheta) - \nabla^{\mathrm{MVG}}_\vtheta J(\vtheta) \|_q 
 \leq &\frac{\gamma \sqrt{2} Z R_{\max}}{(1-\gamma)^2} \sqrt{\E_{s,a \sim \eta_\mu^{\pi,p}} \left[ D_{KL}(p(\cdot|s,a) \| \widehat{p} (\cdot|s,a)) \right]} \\
 \leq&
        \frac{\gamma \sqrt{2} K R_{\max}}{(1-\gamma)^2} \sqrt{\E_{s,a \sim \delta_\mu^{\pi,p}} \left[ D_{KL}(p(\cdot|s,a) \| \widehat{p} (\cdot|s,a)) \right]}.
\end{align*}

This reflects the fact that the standard approach for model learning in MBRL does not make use of all the available information, in this case related to the gradient of the current agent policy.

\subsection{Proofs of Section~\ref{sec:gamps}}
We start introducing the following lemma that states that taking expectations \wrt $\delta_\mu^{\pi,p}$ is equivalent to taking proper expectations \wrt $\zeta^{\pi,p}_{\mu}$.

\begin{lemma}\label{thr:lemmaTraj}
	Let $f : \mathcal{S \times A} \rightarrow \mathbb{R}^k$ an arbitrary function defined over the state-action space. Then, it holds that:
	\begin{equation}
		\E_{s,a \sim \delta_\mu^{\pi,p}} \left[f(s,a)\right] = (1-\gamma)  \sum_{t=0}^{+\infty} \gamma^t  \E_{\tau_{0:t} \sim \zeta^{\pi,p}_{\mu}} \left[ f(s_t,a_t) \right] = (1-\gamma) \E_{\tau \sim \zeta^{\pi,p}_{\mu}} \left[ \sum_{t=0}^{+\infty} \gamma^t  f(s_t,a_t) \right].
	\end{equation}
\end{lemma}

\begin{proof}
We denote with $\mathcal{T}$ the set of all possible trajectories.	We just apply the definition of $\delta_\mu^{\pi,p}$~\cite{sutton2000policy}:
	\begin{align*}
		\E_{s,a \sim \delta_\mu^{\pi,p}} \left[f(s,a)\right] & = \ints \inta \delta_\mu^{\pi,p}(s,a) f(s,a)  \d s \d a \\
			& = (1-\gamma) \sum_{t=0}^{+\infty} \gamma^t \ints \inta \Pr\left(s_t=s,\, a_t=a | \mathcal{M}, \, \pi \right) f(s,a)\d s \d a \\
			& = (1-\gamma) \sum_{t=0}^{+\infty} \gamma^t \ints \inta \left( \intt  \zeta^{\pi,p}_{\mu}(\tau_{0:t}) \mathds{1} \left(s_t=s,\, a_t=a \right) \d \tau_{0:t} \right) f(s,a)\d s \d a \\
			& = (1-\gamma) \sum_{t=0}^{+\infty} \gamma^t \intt  \zeta^{\pi,p}_{\mu}(\tau_{0:t}) \left( \ints \inta \mathds{1} \left(s_t=s,\, a_t=a \right) f(s,a)\d s \d a \right) \d \tau_{0:t} \\
			& = (1-\gamma) \sum_{t=0}^{+\infty} \gamma^t \intt  \zeta^{\pi,p}_{\mu}(\tau_{0:t})  f(s_t,a_t) \d \tau_{0:t} \\
			& = (1-\gamma) \intt  \zeta^{\pi,p}_{\mu}(\tau) \sum_{t=0}^{+\infty} \gamma^t   f(s_t,a_t) \d \tau,
	\end{align*}
	where we exploited the fact that the probability $\Pr\left(s_t=s,\, a_t=a | \mathcal{M}, \, \pi \right)$ is equal to the probability that a prefix of trajectory $\tau_{0:t}$ terminates in $(s_t,a_t)$, \ie $\intt \zeta^{\pi,p}_{\mu}(\tau_{0:t}) \mathds{1} \left(s_t=s,\, a_t=a \right) \d \tau_{0:t}$. The last passage follows from the fact that $f(s_t,a_t)$ depends on random variables realized at time $t$ we can take the expectation over the whole trajectory.
\end{proof}

We can apply this result to rephrase the expectation \wrt $\eta_{\mu}^{\pi,p}$ as an expectation \wrt $\zeta^{\pi,p}_{\mu}$. 

\begin{lemma}\label{thr:lemmaObjTrajPrelim}
	Let $f : \mathcal{S \times A} \rightarrow \mathbb{R}^k$ an arbitrary function defined over the state-action space. Then, it holds that:
	\begin{equation}
		\E_{s,a \sim \eta_\mu^{\pi,p}} \left[f(s,a)\right] = \frac{(1-\gamma)^2}{Z} \E_{\tau \sim \zeta^{\pi,p}_{\mu}} \left[ \sum_{t=0}^{+\infty} \gamma^t   \sum_{l=0}^t \left\| \nabla_{\vtheta} \log \pi (a_l|s_l) \right\|_q  f(s_{t},a_{t}) \right].
	\end{equation}
\end{lemma}

\begin{proof}
	We just need to apply Lemma~\ref{thr:lemmaTraj} twice and exploit the definition of $\eta_\mu^{\pi,p}$:
	\begin{align*}
		\E_{s,a \sim \eta_\mu^{\pi,p}} \left[f(s,a)\right] & = \ints \inta \eta_\mu^{\pi,p}(s,a) f(s,a)  \d s \d a \\
		& = \frac{1}{Z} \ints \inta \ints \inta \delta_\mu^{\pi,p}(s',a') \left\| \nabla_\vtheta \log \pi(a'|s') \right\|_q \delta_{s',a'}^{\pi,p}(s,a) \d s' \d a' f(s,a)  \d s \d a.
	\end{align*}
	Let us first focus on the expectation taken \wrt $\delta_\mu^{\pi,p}(s',a')$. By applying Lemma~\ref{thr:lemmaTraj} with $f(s',a') = \left\| \nabla_\vtheta \log \pi(a'|s') \right\|_q \delta_{s',a'}^{\pi,p}(s,a)$, we have:
	\begin{align*}
	\ints \inta \delta_\mu^{\pi,p}(s',a') & \left\| \nabla_\vtheta \log \pi(a'|s') \right\|_q \delta_{s',a'}^{\pi,p}(s,a) \d s' \d a' \\
	& = (1-\gamma)  \sum_{t=0}^{+\infty} \gamma^t \intt  \zeta^{\pi,p}_{\mu}(\tau_{0:t}) \left\| \nabla_\vtheta \log \pi(a_t|s_t) \right\|_q \delta_{s_t,a_t}^{\pi,p}(s,a) \d \tau_{0:t}.
	\end{align*}
	Now, let us consider $\delta_{s_t,a_t}^{\pi,p}(s,a)$. We instantiate again  Lemma~\ref{thr:lemmaTraj}:
	\begin{align*}
	\E_{s,a \sim \eta_\mu^{\pi,p}} \left[f(s,a)\right] & = \frac{(1-\gamma)}{Z} \ints \inta \sum_{t=0}^{+\infty} \gamma^t \ \intt  \zeta^{\pi,p}_{\mu}(\tau_{0:t}) \left\| \nabla_\vtheta \log \pi(a_t|s_t) \right\|_q \delta_{s_t,a_t}^{\pi,p}(s,a) f(s,a) \d \tau_{0:t} \d s \d a \\
	& =  \frac{(1-\gamma)}{Z}  \sum_{t=0}^{+\infty} \gamma^t  \intt  \zeta^{\pi,p}_{\mu}(\tau_{0:t}) \left\| \nabla_\vtheta \log \pi(a_t|s_t) \right\|_q \ints \inta \delta_{s_t,a_t}^{\pi,p}(s,a)  f(s,a)  \d s \d a \d \tau_{0:t} \\
	& = \frac{(1-\gamma)^2}{Z}  \sum_{t=0}^{+\infty} \gamma^t \intt  \zeta^{\pi,p}_{\mu}(\tau_{0:t}) \left\| \nabla_\vtheta \log \pi(a_t|s_t) \right\|_q \sum_{l=0}^{+\infty} \gamma^l \intt  \zeta^{\pi,p}_{s_t,a_t}(\tau_{0:l}) f(s_l,a_l) \d \tau_{0:l} \d \tau_{0:t} \\
	& = \frac{(1-\gamma)^2}{Z} \intt  \zeta_\mu^{\pi,p}(\tau) \sum_{t=0}^{+\infty} \left\| \nabla_\vtheta \log \pi(a_t|s_t) \right\|_q \sum_{h=t}^{+\infty} \gamma^{h} f(s_{h},a_{h}) \d \tau,
	\end{align*}
	where the last passage derives from observing that, for each $t$ and $l$ we are computing an integral over the trajectory prefixes of length $h\coloneqq t+l$ and observing that $(s_l,a_l)$ can be seen as the $h$-th state-action pair of a trajectory $\tau \sim \zeta_\mu^{\pi,p}$.
	We now rearrange the summations:
	\begin{align*}
	\sum_{t=0}^{+\infty} \left\| \nabla_\vtheta \log \pi(a_t|s_t) \right\|_q \sum_{h=t}^{+\infty} \gamma^{h} f(s_{h},a_{h}) = \sum_{h=0}^{+\infty} \gamma^{h} f(s_{h},a_{h}) \sum_{t=0}^{h} \left\| \nabla_\vtheta \log \pi(a_t|s_t)\right\|_q.
	\end{align*}
	By changing the names of the indexes of the summations, we get the result.
\end{proof}

We are now ready to prove Lemma~\ref{thr:lemmaObjTraj}.
\lemmaObjTraj*
\begin{proof}
	What changes \wrt Lemma~\ref{thr:lemmaObjTrajPrelim} is that we are now interested in computing the expectation \wrt to a target policy $\pi$ while trajectories are collected with a behavioral policy $\pi_b$, fulfilling the hypothesis stated in the lemma. We start from Lemma~\ref{thr:lemmaObjTraj} and we just need to apply importance weighting~\cite{mcbook}:
	\begin{align*}
		 \E_{\tau \sim \zeta^{\pi,p}_{\mu}} & \left[ \sum_{t=0}^{+\infty} \gamma^t   \sum_{l=0}^t \left\| \nabla_{\vtheta} \log \pi (a_l|s_l) \right\|_q  f(s_{t},a_{t}) \right]  =  \sum_{t=0}^{+\infty} \gamma^t \E_{\tau_{0:t} \sim \zeta^{\pi,p}_{\mu}} \left[   \sum_{l=0}^t \left\| \nabla_{\vtheta} \log \pi (a_l|s_l) \right\|_q  f(s_{t},a_{t}) \right] \\
		& = \sum_{t=0}^{+\infty} \gamma^t \E_{\tau_{0:t} \sim \zeta^{\pi_b,p}_{\mu}} \left[ \frac{\zeta_\mu^{\pi,p}(\tau_{0:t})}{\zeta_\mu^{\pi_b,p}(\tau_{0:t})}  \sum_{l=0}^t \left\| \nabla_{\vtheta} \log \pi (a_l|s_l) \right\|_q  f(s_{t},a_{t}) \right] \\
		& = \sum_{t=0}^{+\infty} \gamma^t \E_{\tau_{0:t} \sim \zeta^{\pi_b,p}_{\mu}} \left[ \rho_{\pi/\pi_b}(\tau_{0:t}) \sum_{l=0}^t \left\| \nabla_{\vtheta} \log \pi (a_l|s_l) \right\|_q  f(s_{t},a_{t}) \right] \\
		& = \E_{\tau \sim \zeta^{\pi_b,p}_{\mu}} \left[ \sum_{t=0}^{+\infty} \gamma^t   \rho_{\pi/\pi_b}(\tau_{0:t}) \sum_{l=0}^t \left\| \nabla_{\vtheta} \log \pi (a_l|s_l) \right\|_q  f(s_{t},a_{t}) \right].
	\end{align*}
\end{proof}

%

\subsection{Details about Assumption~\ref{ass:boundedMoment}}\label{apx:assumption}
Assumption~\ref{ass:boundedMoment} is equivalent to require that there exists two finite constants $c_1 < +\infty$ and $c_2 < +\infty$ such that for all $\overline{p} \in \mathcal{P}$ and $\pi \in \Pi_\Theta$::
\begin{align}
	&\E_{\tau \sim \zeta^{\pi_b,p}_{\mu} } \left[ \left( \sum_{t=0}^{+\infty} \gamma^t  \rho_{\pi/\pi_b}(\tau_{0:t}) \sum_{l=0}^t \left\| \nabla_{\vtheta} \log \pi (a_l|s_l) \right\|_q \log p(s_{t+1}|s_t,a_t) \right)^2 \right] \le c_1^2, \label{eq:constantAssumption}\\
	&\E_{\tau \sim \zeta^{\pi_b,p}_{\mu} } \left[ \left( \sum_{t=0}^{+\infty} \gamma^t  \rho_{\pi/\pi_b}(\tau_{0:t})  \nabla_{\vtheta_j} \log \pi (a_t|s_t) Q^{\pi,\overline{p}} (s_t,a_t) \right)^2 \right] \le R_{\max}^2 c_2^2, \quad j=1,...,d.\label{eq:constantAssumption2}
\end{align}

We now state the following result that allows decoupling Assumption~\ref{ass:boundedMoment} into two separate conditions for the policies $\pi$ and $\pi_b$ and the transition models $p$ (the real one) and $\overline{p}$ (the approximating one).

\begin{corollary}
	Assumption~\ref{ass:boundedMoment} is satisfied if there exist three constants $\chi_1$, $\chi_2$ and $\chi_3$, with $\chi_1 < \frac{1}{\gamma}$.
	\begin{align*}
		& \sup_{\pi \in \Pi_\Theta} \sup_{s \in \mathcal{S}} \E_{a \sim \pi_b(\cdot|s)} \left[\left(\frac{\pi(a|s)}{\pi_b(a|s)} \right)^2 \right] \le \chi_1, \\
		& \sup_{\pi \in \Pi_\Theta} \sup_{s \in \mathcal{S}} \E_{a \sim \pi_b(\cdot|s)} \left[\left(\frac{\pi(a|s)}{\pi_b(a|s)} \left\| \nabla_\vtheta \log \pi(a|s) \right\|_q^2 \right)^2 \right] \le \chi_2, \\
		& \sup_{\overline{p} \in \mathcal{P}} \sup_{\substack{s \in \mathcal{S} \\a \in \mathcal{A}}} \E_{s' \sim p(\cdot|s,a)} \left[\left( \log \overline{p}(s'|s,a) \right)^2 \right] \le \chi_3.
	\end{align*}
	In such case, Equation~\eqref{eq:constantAssumption} and Equation~\eqref{eq:constantAssumption2} are satisfied with constants:
	\begin{equation*}
		c_1^2 =  \frac{\chi_3 \chi_2(1+\gamma\chi_1)}{(1-\gamma)(1-\gamma\chi_1)^3}, \quad c_2^2 = \frac{\chi_3 \chi_2}{(1-\gamma)^3(1-\gamma\chi_1)}.
	\end{equation*}
\end{corollary}

\begin{proof}
Let us start with Equation~\eqref{eq:constantAssumption}. We first apply Cauchy Swartz inequality to bring the expectation inside the summation:
\begin{align*}
	\E_{\tau \sim \zeta^{\pi_b,p}_{\mu} } & \left[ \left( \sum_{t=0}^{+\infty} \gamma^\frac{t}{2} \cdot \gamma^\frac{t}{2}  \rho_{\pi/\pi_b}(\tau_{0:t}) \sum_{l=0}^t \left\| \nabla_{\vtheta} \log \pi (a_l|s_l) \right\|_q \log p(s_{t+1}|s_t,a_t) \right)^2 \right] \\
	& \le \sum_{t=1}^{+\infty} \gamma^t \E_{\tau \sim \zeta^{\pi_b,p}_{\mu} } \left[  \sum_{t=0}^{+\infty} \gamma^t \left( \rho_{\pi/\pi_b}(\tau_{0:t}) \sum_{l=0}^t \left\| \nabla_{\vtheta} \log \pi (a_l|s_l) \right\|_q \log p(s_{t+1}|s_t,a_t) \right)^2 \right] \\
	& \le \frac{1}{1-\gamma}   \sum_{t=0}^{+\infty} \gamma^t \E_{\tau \sim \zeta^{\pi_b,p}_{\mu} } \left[ \left( \rho_{\pi/\pi_b}(\tau_{0:t}) \sum_{l=0}^t \left\| \nabla_{\vtheta} \log \pi (a_l|s_l) \right\|_q \log p(s_{t+1}|s_t,a_t) \right)^2 \right].
\end{align*} 
	Let us fix a timestep $t$. We derive the following bound:
	\begin{align*}
	\E_{\tau \sim \zeta^{\pi_b,p}_{\mu} }  & \left[\left( \rho_{\pi/\pi_b}(\tau_{0:t}) \sum_{l=0}^t \left\| \nabla_{\vtheta} \log \pi (a_l|s_l) \right\|_q \log p(s_{t+1}|s_t,a_t) \right)^2 \right]\\
	&  = \E_{\tau \sim \zeta^{\pi_b,p}_{\mu} }  \left[ \left( \sum_{l=0}^t \rho_{\pi/\pi_b}(\tau_{0:t}) \left\| \nabla_{\vtheta} \log \pi (a_l|s_l) \right\|_q  \right)^2 \left(\log p(s_{t+1}|s_t,a_t) \right)^2 \right] \\
	& \le \E_{\tau \sim \zeta^{\pi_b,p}_{\mu} }  \left[ (t+1) \sum_{l=0}^t \left( \rho_{\pi/\pi_b}(\tau_{0:t}) \left\| \nabla_{\vtheta} \log \pi (a_l|s_l) \right\|_q  \right)^2 \left(\log p(s_{t+1}|s_t,a_t) \right)^2 \right], 
	\end{align*}
	where we applied Cauchy-Swartz inequality to bound the square of the summation. We now rewrite the expectation in a convenient form to highlight the different components. 
	\begin{align*}
	\E_{\tau \sim \zeta^{\pi_b,p}_{\mu} }  & \left[ (t+1) \sum_{l=0}^t \left( \rho_{\pi/\pi_b}(\tau_{0:t}) \left\| \nabla_{\vtheta} \log \pi (a_l|s_l) \right\|_q  \right)^2 \left(\log p(s_{t+1}|s_t,a_t) \right)^2 \right] \\
	& = (t+1) \sum_{l=0}^t \E_{\tau_{0:t} \sim \zeta^{\pi_b,p}_{\mu} }  \left[ \left( \rho_{\pi/\pi_b}(\tau_{0:t}) \left\| \nabla_{\vtheta} \log \pi (a_l|s_l) \right\|_q  \right)^2 \E_{s_{t+1} \sim p(\cdot|s_t,a_t)} \left[ \left(\log p(s_{t+1}|s_t,a_t) \right)^2 \right] \right]\\
	& \le (t+1) \chi_3 \sum_{l=0}^t \E_{\tau_{0:t} \sim \zeta^{\pi_b,p}_{\mu} }  \left[ \left( \rho_{\pi/\pi_b}(\tau_{0:t}) \left\| \nabla_{\vtheta} \log \pi (a_l|s_l) \right\|_q  \right)^2 \right].
	\end{align*}
	Let us fix $l$ and bound the expectation inside the summation, by unrolling the trajectory and recalling the definition of $\rho_{\pi/\pi_b}(\tau_{0:t})$:
	\begin{align*}
	\E_{\tau_{0:t} \sim \zeta^{\pi_b,p}_{\mu} } & \left[ \left( \rho_{\pi/\pi_b}(\tau_{0:t}) \left\| \nabla_{\vtheta} \log \pi (a_l|s_l) \right\|_q  \right)^2 \right] \\
	& = \E_{\substack{s_0 \sim \mu \\a_0 \sim \pi(\cdot|s_0)}} \Bigg[ \left(\frac{\pi(a_0|s_0)}{\pi_b(a_0|s_0)} \right)^2 \E_{\substack{s_1 \sim p(\cdot|s_0,a_0) \\a_1 \sim \pi(\cdot|s_1)}} \Bigg[ \left(\frac{\pi(a_1|s_1)}{\pi_b(a_1|s_1)} \right)^2  \dots \\
	& \quad \times \E_{\substack{s_l \sim p(\cdot|s_{l-1},a_{l-1}) \\a_l \sim \pi(\cdot|s_l)}} \Bigg[ \left(\frac{\pi(a_l|s_l)}{\pi_b(a_l|s_l)} \left\| \nabla_{\vtheta} \log \pi (a_l|s_l) \right\|_q \right)^2  \dots \E_{\substack{s_t \sim p(\cdot|s_{t-1},a_{t-1}) \\a_t \sim \pi(\cdot|s_t)}} \Bigg[ \left(\frac{\pi(a_t|s_t)}{\pi_b(a_t|s_t)} \right)^2  \Bigg] \dots \Bigg] \dots \Bigg] \Bigg]\\
	& \le  \chi_2 \chi_1^{t}.
	\end{align*}
	Plugging this result in the summation we get the result, recalling that $\gamma\chi_1 < 1$ and using the properties of the geometric series, we obtain:
	\begin{equation*}
		\frac{1}{1-\gamma} \sum_{t=0}^{+\infty} (t+1)^2 \gamma^t \chi_1^{t} \chi_2 \chi_3 = \frac{\chi_3 \chi_2(1+\gamma\chi_1)}{(1-\gamma)(1-\gamma\chi_1)^3}.
	\end{equation*}
	We now consider Equation~\eqref{eq:constantAssumption2} and we apply Cauchy Swartz as well:
	\begin{align*}
	\E_{\tau \sim \zeta^{\pi_b,p}_{\mu} } & \left[ \left( \sum_{t=0}^{+\infty} \gamma^{\frac{t}{2}}  \cdot \gamma^{\frac{t}{2}} \rho_{\pi/\pi_b}(\tau_{0:t})  \nabla_{\vtheta_j} \log \pi (a_t|s_t) Q^{\pi,\overline{p}} (s_t,a_t) \right)^2 \right]\\
	&  \le \frac{1}{1-\gamma} \sum_{t=0}^{+\infty} \gamma^{t} \E_{\tau \sim \zeta^{\pi_b,p}_{\mu} } \left[\left(\rho_{\pi/\pi_b}(\tau_{0:t})  \nabla_{\vtheta_j} \log \pi (a_t|s_t) Q^{\pi,\overline{p}} (s_t,a_t) \right)^2 \right].
\end{align*} 
By observing that $\left| Q^{\pi,\overline{p}} (s_t,a_t)\right| \le \frac{R_{\max}}{1-\gamma}$ and $\left|\nabla_{\vtheta_j} \log \pi (a_t|s_t) \right| \le \left\|\nabla_{\vtheta} \log \pi (a_t|s_t) \right\|_{\infty} \le \left\|\nabla_{\vtheta} \log \pi (a_t|s_t) \right\|_{q}$ we can use an argument similar to the one used to bound Equation~\eqref{eq:constantAssumption} to get:
\begin{align*}
	\E_{\tau \sim \zeta^{\pi_b,p}_{\mu} } \left[\left(\rho_{\pi/\pi_b}(\tau_{0:t})  \nabla_{\vtheta_j} \log \pi (a_t|s_t) Q^{\pi,\overline{p}} (s_t,a_t) \right)^2 \right] \le \frac{R_{\max}^2}{(1-\gamma)^2} \chi_2 \chi_1^t.
\end{align*}
Plugging this result into the summation, we have:
\begin{equation*}
\frac{1}{1-\gamma} \sum_{t=0}^{+\infty}  \gamma^t \chi_1^{t} \chi_2 \chi_3 \frac{R_{\max}^2}{(1-\gamma)^2} = \frac{\chi_3 \chi_2 R_{\max}^2}{(1-\gamma)^3(1-\gamma\chi_1)}.
\end{equation*}
\end{proof}

\subsection{Proofs of Section~\ref{sec:theoreticalAnalysis}}
Under Assumption~\ref{ass:boundedMoment}, we prove the following intermediate result about the objective function in Equation~\eqref{eq:objectiveP}.

\begin{lemma}\label{thr:ltP}
Let $\widehat{p} \in \mathcal{P}$ be the maximizer of the objective function in Equation~\eqref{eq:objectiveP}, obtained with $N>0$ independent trajectories $\{\tau^i\}_{i=1}^N$. Under Assumption~\ref{ass:boundedMoment} and~\ref{ass:boundedPdim}, for any $\delta \in (0,1)$, with probability at least $1-2\delta$ it holds that:
	\begin{equation}
		\E_{\tau \sim \zeta_{\mu}^{\pi_b,p}} \left[ l^{\pi,\widehat{p}}(\tau)\right] \ge \sup_{\overline{p} \in \mathcal{P}} \E_{\tau \sim \zeta_{\mu}^{\pi_b,p}} \left[l^{\pi,\overline{p}}(\tau) \right] - 4 c_1 \epsilon,
	\end{equation}
	where $\epsilon = \sqrt{\frac{v \log \frac{2eN}{v} + \log \frac{4}{\delta}}{N}} \Gamma \left(\sqrt{\frac{v \log \frac{2eN}{v} + \log \frac{4}{\delta}}{N}} \right)$ and $\Gamma(\xi) \coloneqq \frac{1}{2} + \sqrt{ 1 + \frac{1}{2} \log \frac{1}{\xi} }  = \widetilde{\mathcal{O}}(1)$.
\end{lemma}

\begin{proof}
	We use a very common argument of empirical risk minimization. Let us denote with $\widetilde{p} \in \argmax_{\overline{p} \in \mathcal{P}} \E_{\tau \sim \zeta_{\mu}^{\pi_b,p}} \left[l^{\pi,\overline{p}}(\tau) \right]$ and $\widehat{\mathcal{L}}^{\pi,\overline{p}} = \frac{1}{N} \sum_{i=1}^N l^{\pi,\overline{p}}(\tau^i)$:
	\begin{align*}
		\E_{\tau \sim \zeta_{\mu}^{\pi_b,p}} \left[ l^{\pi,\widehat{p}}(\tau)\right] - \E_{\tau \sim \zeta_{\mu}^{\pi_b,p}} \left[l^{\pi,\widetilde{p}}(\tau) \right] & = \E_{\tau \sim \zeta_{\mu}^{\pi_b,p}} \left[ l^{\pi,\widehat{p}}(\tau)\right] - \E_{\tau \sim \zeta_{\mu}^{\pi_b,p}} \left[l^{\pi,\widetilde{p}}(\tau) \right] \pm \widehat{\mathcal{L}}^{\pi,\widehat{p}} \\
		& \ge \E_{\tau \sim \zeta_{\mu}^{\pi_b,p}} \left[ l^{\pi,\widehat{p}}(\tau)\right] - \widehat{\mathcal{L}}^{\pi,\widehat{p}}  -  \E_{\tau \sim \zeta_{\mu}^{\pi_b,p}} \left[l^{\pi,\widetilde{p}}(\tau) \right] + \widehat{\mathcal{L}}^{\pi,\widetilde{p}}\\
		& \ge - 2 \sup_{\overline{p} \in \mathcal{P}} \left| \widehat{\mathcal{L}}^{\pi,\overline{p}} - \E_{\tau \sim \zeta_{\mu}^{\pi_b,p}} \left[ l^{\pi,\overline{p}}(\tau)\right] \right|,
	\end{align*}
	where we exploited the fact that $\widehat{\mathcal{L}}^{\pi,\widetilde{p}} \le \widehat{\mathcal{L}}^{\pi,\widehat{p}}$, as $\widehat{p}$ is the maximizer of $\widehat{\mathcal{L}}^{\pi,\cdot}$. The result follows from the application of Corollary 14 in~\cite{cortes2013relative}, having bounded the growth function with the pseudodimension, as in Corollary 18 of~\cite{cortes2013relative}.
\end{proof}

We can derive a concentration result for the gradient estimation (Equation~\eqref{eq:gradient_estimate}), recalling the fact that $\mathbr{g}^{\pi,\overline{p}}$ is a vectorial function.
\begin{lemma}\label{thr:ltG}
Let $q \in [1,+\infty]$, $d$ be the dimensionality of $\Theta$ and $\widehat{p} \in \mathcal{P}$ be the maximizer of the objective function in Equation~\eqref{eq:objectiveP}, obtained with $N>0$ independent trajectories $\{\tau^i\}_{i=1}^N$. Under Assumption~\ref{ass:boundedMoment} and~\ref{ass:boundedPdim}, for any $\delta \in (0,1)$, with probability at least $1-2d\delta$, simultaneously for all $\overline{p} \in \mathcal{P}$, it holds that:
	\begin{equation}
		\left\| \widehat{\nabla}_{\vtheta} J(\vtheta) - \nabla_{\vtheta}^{\mathrm{MVG}} J(\vtheta) \right\|_q \le
		 2 d^{\frac{1}{q}} R_{\max} c_2  \epsilon,
	\end{equation}
	where $\epsilon = \sqrt{\frac{v \log \frac{2eN}{v} + \log \frac{4}{\delta}}{N}} \Gamma \left(\sqrt{\frac{v \log \frac{2eN}{v} + \log \frac{4}{\delta}}{N}} \right)$ and $\Gamma(\xi) \coloneqq \frac{1}{2} + \sqrt{ 1 + \frac{1}{2} \log \frac{1}{\xi} }  = \widetilde{\mathcal{O}}(1)$.
\end{lemma}

\begin{proof}
We observe that $\widehat{\nabla}_{\vtheta} J(\vtheta)$ is the sample version of $\nabla_{\vtheta}^{\mathrm{MVG}} J(\vtheta)$. Under Assumption~\ref{ass:boundedMoment} and~\ref{ass:boundedPdim}, and using Corollary 14 in~\cite{cortes2013relative} as in Lemma~\ref{thr:ltP}, we can write for any $j=1,...d$ the following bound that holds with probability at least $1-2\delta$, simultaneously for all $\widehat{p} \in \mathcal{P}$:
	\begin{equation}
		\left| \widehat{\nabla}_{{\vtheta}_j} J(\vtheta) - \nabla_{{\vtheta}_j}^{\mathrm{MVG}} J(\vtheta) \right| \le
		 2 R_{\max} c_2  \epsilon.
	\end{equation}
	Considering the $L^q$-norm, and plugging the previous equation, we have that with probability at least $1-2d\delta$ it holds that, simultaneously for all $\widehat{p} \in \mathcal{P}$:
	\begin{align*}
	\left\| \widehat{\nabla}_{\vtheta} J(\vtheta) - \nabla_{\vtheta}^{\mathrm{MVG}} J(\vtheta) \right\|_q & = \left( \sum_{j=1}^d \left| \nabla_{{\vtheta}_j}^{\mathrm{MVG}} J(\vtheta) - \nabla_{{\vtheta}_j} J(\vtheta)\right|^q \right)^{\frac{1}{q}} \le 2 d^{\frac{1}{q}} R_{\max} c_2  \epsilon,
	\end{align*}
	having exploited a union bound over the dimensions $d$.
\end{proof}

%

We are now ready to prove the main result.
\mainTheorem*
\begin{proof}
	Let us first consider the decomposition, that follows from triangular inequality:
	\begin{align*}
		\left\| \widehat{\nabla}_{\vtheta} J(\vtheta) - \nabla_{\vtheta} J(\vtheta) \right\|_q & = \left\| \widehat{\nabla}_{\vtheta} J(\vtheta) - \nabla_{\vtheta} J(\vtheta) \pm \nabla_{\vtheta}^{\mathrm{MVG}} J(\vtheta) \right\|_q \\
		& \le \underbracket{\left\| \widehat{\nabla}_{\vtheta} J(\vtheta) - \nabla_{\vtheta}^{\mathrm{MVG}} J(\vtheta) \right\|_q}_{\text{(i)}} + \underbracket{\left\| \nabla_{\vtheta}^{\mathrm{MVG}} J(\vtheta) - \nabla_{\vtheta} J(\vtheta)\right\|_q}_{\text{(ii)}}.
	\end{align*}
	We now bound each term of the right hand side. (i) is bounded in Lemma~\ref{thr:ltG}. Let us now consider (ii). We just need to apply Theorem~\ref{thr:lemmaObjTraj} and Lemma~\ref{thr:ltP}, recalling the properties of the KL-divergence. From Theorem~\ref{th:weighting}:
	\begin{align}
		 \Big\| & {\nabla}_{\vtheta}^{\text{MVG}} J(\vtheta) - \nabla_{\vtheta} J(\vtheta) \Big\|_q  \leq  \frac{\gamma \sqrt{2} Z R_{\max}}{(1-\gamma)^2} \sqrt{\E_{s,a \sim \eta^{\pi,p}_{\mu}} \left[ D_{KL}(p(\cdot|s,a) \| \widehat{p} (\cdot|s,a)) \right] } \notag\\
		& = \frac{\gamma \sqrt{2} Z R_{\max}}{(1-\gamma)^2} \sqrt{\E_{s,a \sim \eta^{\pi,p}_{\mu}} \left[ \ints p(s'|s,a) \log p(s'|s,a) \d s' - \ints  p(s'|s,a) \log \widehat{p}(s'|s,a) \d s' \right] }\label{p:1}\\
		& = \frac{\gamma \sqrt{2Z} R_{\max}}{(1-\gamma)} \sqrt{\E_{\tau \sim \zeta_\mu^{\pi_b,p}}\left[\sum_{t=0}^{+\infty} \omega_t \left(  \log p(s_{t+1}|s_t,a_t)  - \log \widehat{p}(s_{t+1}|s_t,a_t) \right) \right]}\label{p:2} \\
		& = \frac{\gamma \sqrt{2Z} R_{\max}}{(1-\gamma)} \sqrt{\E_{\tau \sim \zeta_\mu^{\pi_b,p}}\left[  \sum_{t=0}^{+\infty}  \omega_t \log p(s_{t+1}|s_t,a_t) \right] -  \E_{\tau \sim \zeta_\mu^{\pi_b,p}}\left[l^{\pi,\widehat{p}}(\tau) \right] } \notag\\
		& \le \frac{\gamma \sqrt{2Z} R_{\max}}{(1-\gamma)} \sqrt{\E_{\tau \sim \zeta_\mu^{\pi_b,p}}\left[  \sum_{t=0}^{+\infty}  \omega_t \log p(s_{t+1}|s_t,a_t) \right] -  \sup_{\overline{p} \in \mathcal{P}} \E_{\tau \sim \zeta_{\mu}^{\pi_b,p}} \left[l^{\pi,\overline{p}}(\tau) \right] + 4 c_1 \epsilon }\label{p:3}\\
		& = \frac{\gamma \sqrt{2Z} R_{\max}}{(1-\gamma)} \sqrt{\inf_{\overline{p} \in \mathcal{P}} \E_{\tau \sim \zeta_\mu^{\pi_b,p}}\left[  \sum_{t=0}^{+\infty}  \omega_t \left(  \log p(s_{t+1}|s_t,a_t)  - \log \overline{p}(s_{t+1}|s_t,a_t) \right) \right] + 4 c_1 \epsilon }\\
		& \le \frac{\gamma \sqrt{2Z} R_{\max}}{(1-\gamma)} \sqrt{\inf_{\overline{p} \in \mathcal{P}} \E_{\tau \sim \zeta_\mu^{\pi_b,p}}\left[  \sum_{t=0}^{+\infty}  \omega_t \left(  \log p(s_{t+1}|s_t,a_t)  - \log \overline{p}(s_{t+1}|s_t,a_t) \right) \right]} +  \frac{2 \gamma   R_{\max} \sqrt{2 Z c_1 \epsilon}}{1 - \gamma} \label{p:34}\\
		& = \frac{\gamma \sqrt{2} Z R_{\max}}{(1-\gamma)^2} \sqrt{\inf_{\overline{p} \in \mathcal{P}} \E_{s,a \sim \eta^{\pi,p}_{\mu}} \left[ D_{KL}(p(\cdot|s,a) \| \overline{p} (\cdot|s,a)) \right]} + \frac{2 \gamma   R_{\max} \sqrt{2 Z c_1 \epsilon}}{1 - \gamma} ,\label{p:4}
	\end{align}
	where Equation~\eqref{p:1} and Equation~\eqref{p:4} follow from the definition of KL-divergence and Lemma~\ref{thr:lemmaObjTraj}. Equation~\eqref{p:2} is derived from Lemma~\ref{thr:lemmaObjTraj} where $\omega_t = \gamma^t \rho_{\pi/\pi_b}(\tau_{0:t}) \sum_{l=0}^{t} \left\| \nabla_\vtheta \log \pi(a_l|s_l) \right\|_q$. Equation~\eqref{p:3} is obtained by applying Lemma~\ref{thr:ltP}. Equation~\eqref{p:34} follows from the subadditivity of the square root.
	Putting together (i) and (ii) we get the result that holds with probability at least $1-2(d+1)\delta$ as bound (i) holds w.p. $1-2\delta$ and bound (ii) w.p. $1-2d\delta$. By rescaling $\delta$ we get the result.
\end{proof}

\section{Experimental details}
In this appendix, we report an extensive explanation of the domains employed in the experimental evaluation along with some details on the policy and approximating transition models employed.

\subsection{Two-areas Gridworld}
The gridworld we use in our experiments features two subspaces of the state space $\mathcal{S}$, to which we refer to as  $\mathcal{S}_1$ (\emph{lower}) and $\mathcal{S}_2$ (\emph{upper}).

The agent can choose among four different actions: in the lower part, a sticky area, each action corresponds to an attempt to go up, right, down or left, and has a $0.9$ probability of success and a $0.1$ probability of causing the agent to remain in the same state; in the upper part, the four actions have deterministic movement effects, all different from the ones they have in the other area (rotated of 90 degrees). Representing as $(\underset{\Uparrow}{p_1}, \underset{\Rightarrow}{p_2}, \underset{\Downarrow}{p_3}, \underset{\Leftarrow}{p_4}, \underset{	\nLeftrightarrow}{p_5})$ the probabilities $p_1,p_2,p_3,p_4$ and $p_5$ of, respectively, going up, right, down, left and remaining in the same state, the transition model of the environment is defined as follows:

\begin{align*}
  &s \in \mathcal{S}_1: p(\cdot|s,a)=\left\{
  \begin{array}{@{}ll@{}}
(\underset{\Uparrow}{0}, \underset{\Rightarrow}{0.9}, \underset{\Downarrow}{0}, \underset{\Leftarrow}{0}, \underset{	\nLeftrightarrow}{0.1}), & \text{if}\ a=0 \\[15pt]
    (\underset{\Uparrow}{0}, \underset{\Rightarrow}{0}, \underset{\Downarrow}{0.9}, \underset{\Leftarrow}{0}, \underset{	\nLeftrightarrow}{0.1}), & \text{if}\ a=1 \\[15pt]
    (\underset{\Uparrow}{0}, \underset{\Rightarrow}{0}, \underset{\Downarrow}{0}, \underset{\Leftarrow}{0.9}, \underset{	\nLeftrightarrow}{0.1}), & \text{if}\ a=2 \\[15pt]
    (\underset{\Uparrow}{0.9}, \underset{\Rightarrow}{0}, \underset{\Downarrow}{0}, \underset{\Leftarrow}{0}, \underset{	\nLeftrightarrow}{0.1}), & \text{if}\ a=3
  \end{array}\right.,\\
  \\
  &s \in \mathcal{S}_2: p(\cdot|s,a)=\left\{
  \begin{array}{@{}ll@{}}
    (\underset{\Uparrow}{1}, \underset{\Rightarrow}{0}, \underset{\Downarrow}{0}, \underset{\Leftarrow}{0}, \underset{	\nLeftrightarrow}{0}), & \text{if}\ a=0 \\[15pt]
    (\underset{\Uparrow}{0}, \underset{\Rightarrow}{1}, \underset{\Downarrow}{0}, \underset{\Leftarrow}{0}, \underset{	\nLeftrightarrow}{0}), & \text{if}\ a=1 \\[15pt]
    (\underset{\Uparrow}{0}, \underset{\Rightarrow}{0}, \underset{\Downarrow}{1}, \underset{\Leftarrow}{0}, \underset{	\nLeftrightarrow}{0}), & \text{if}\ a=2 \\[15pt]
    (\underset{\Uparrow}{0}, \underset{\Rightarrow}{0}, \underset{\Downarrow}{0}, \underset{\Leftarrow}{1}, \underset{	\nLeftrightarrow}{0}), & \text{if}\ a=3
  \end{array}\right..
\end{align*} 

There is a reward of -1 in all states apart a single absorbing goal state, located on the upper left corner, that yields zero reward.
The initial state is uniformly chosen among the ones on the low and right border and the agent cannot go back to the sticky part once it reached the second area, in which it passes through the walls to get to the other side.

As policy class $\Pi_\Theta$, we use policies linear in the one-hot representation of the current state. The policy outputs a Boltzman probability distribution over the four possible actions. 
In the lower part of the environment, we initialize the policy as deterministic: the agent tries to go up as long as it can, and goes left when a wall is encountered. Being the policy deterministic for these actions, the corresponding score is zero.

As model class $\mathcal{P}$, we employ the one in which each $\widehat{p} \in \mathcal{P}$ is such that $\widehat{p}(m|s,a) = \text{softmax}(\mathds{1}_a^T \mathbf{W})$, where $\mathbf{W}$ is a matrix of learnable parameters, $\mathds{1}_a$ is the one-hot representation of the action and $m \in  \left\{ \Uparrow, \Rightarrow, \Downarrow, \Leftarrow, \nLeftrightarrow \right\}$ is a movement effect.
This model class has very little expressive power: the forward model is, in practice, executing a probabilistic lookup using the current actions, trying to guess what the next state is. 

We learn both the policy and the models by minimizing the corresponding loss function via gradient descent. We use the Adam optimizer \cite{adam} with a learning rate of 0.2 for the former and of 0.01 for the latter, together with $\beta_1=0.9$ and $\beta_2=0.999$. These hypeparameters were chosen by trial and error from a range of $(0.001, 0.9)$.

In order to understand the properties of our method for model learning, we compare the maximum likelihood model (ML) and the one obtained with GAMPS, in terms of accuracy in next state prediction and MSE with the real Q-function w.r.t. to the one derived by dynamic programming; lastly, we use the computed action-value functions to provide two approximations to the sample version of Equation~\ref{eq:gradient_approximation}. The intuitive rationale behind decision-aware model learning is that the raw quality of the estimate of the forward model itself or any intermediate quantity is pointless: the accuracy on estimating the quantity of interest for improving the policy, in our case its gradient, is the only relevant metric. The results, shown in Table~\ref{tab:estimation_results}, illustrate exactly this point, showing that, although our method offers worse performance in model and Q-function estimation, it is able to perfectly estimate the correct direction of the policy gradient. The definitions of the metrics used for making the comparison, computed over an hold-out set of 1000 validation trajectories, are now presented.
The model accuracy for an estimated model $\hat{p}$ is defined as $\text{acc}(\widehat{p}) \in \frac{1}{|\mathcal{D}|} \sum_{(s,a,s') \in \mathcal{D}} \mathds{1}(s' = \argmax_{\overline{s}} \widehat{p}(\overline{s}|s,a)) $.
The MSE for measuring the error in estimating the tabular Q-function is computed by averaging the error obtained for every state and action.
Lastly, the cosine similarity between the real gradient $\nabla_\vtheta J(\vtheta)$ and the estimated gradient $\widehat{\nabla}_\vtheta J(\vtheta)$ is defined as $\text{sim}(\nabla_\vtheta J(\vtheta), \widehat{\nabla}_\vtheta J(\vtheta))= \frac{\nabla_\vtheta J(\vtheta) \cdot \widehat{\nabla}_\vtheta J(\vtheta)}{\max(\|\nabla_\vtheta J(\vtheta)\|_2 \cdot \|  \widehat{\nabla}_\vtheta J(\vtheta)\|_2, \epsilon)}$, where $\epsilon$ is set to $10^{-8}$.

\begin{table}[t]
\label{tab:estimation_results}
\small
\centering
\caption{Estimation performance on the gridworld environment comparing Maximum Likelihood estimation (ML) and our approach (GAMPS). 1000 training and 1000 validation trajectories per run. Average results on 10 runs with a 95\% confidence interval.}
\begin{tabular}[t]{cccc}
\toprule
\label{tab:estimation_results}
Approach &$\widehat{p}$ accuracy&$\widehat{Q}$ MSE & $\widehat {\nabla}_\vtheta J$ cosine similarity\\
\midrule
ML & $0.765 \pm 0.001$ & $11.803 \pm 0.158$ & $0.449 \pm 0.041$ \\
GAMPS & $0.357 \pm 0.004$ & $633.835 \pm 12.697$ & $1.000 \pm 0.000$ \\
\bottomrule
\end{tabular}
\end{table}

\begin{figure}[t]
  \begin{center}
    \includegraphics[width=0.4\textwidth]{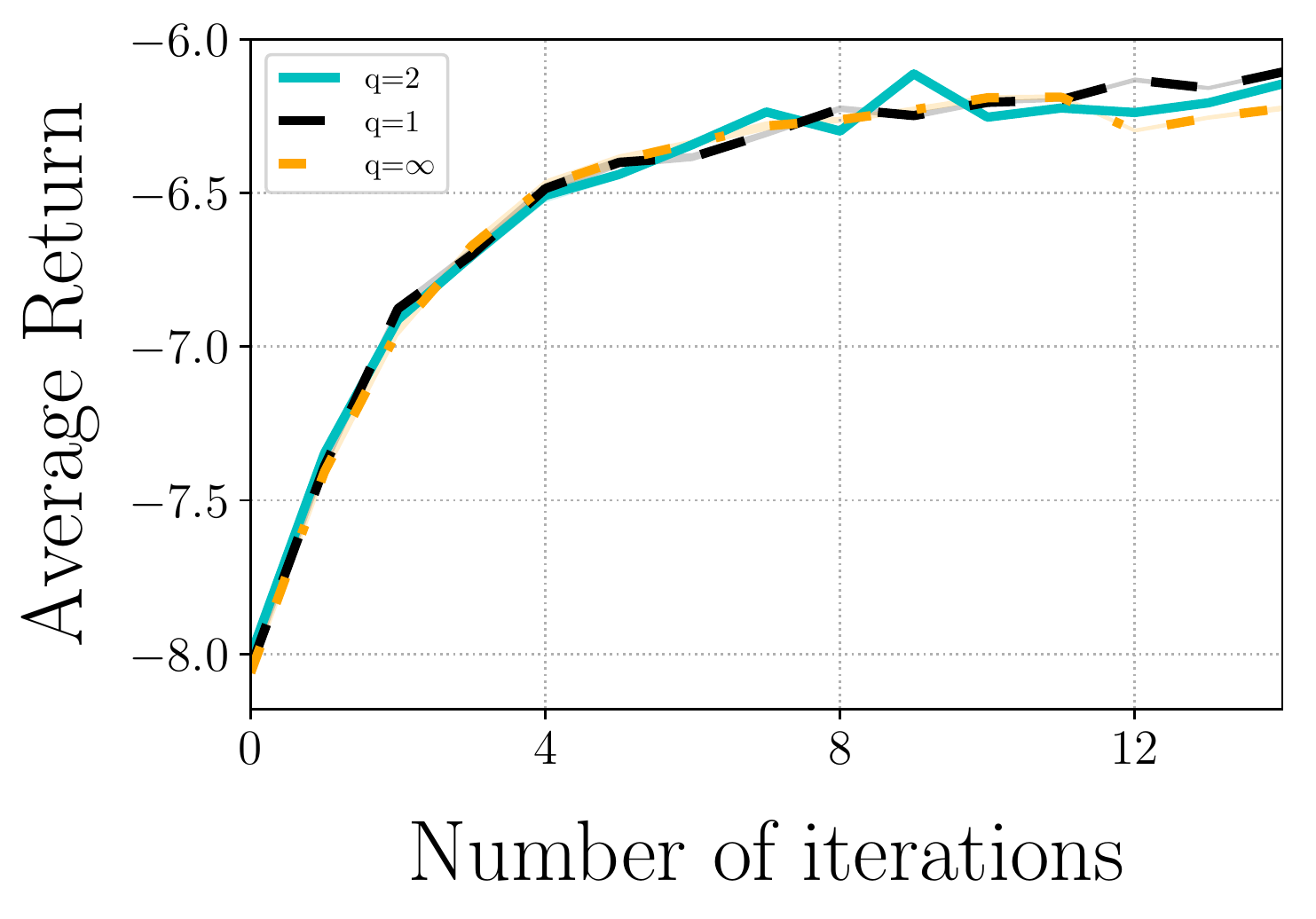}
  \end{center}
  \caption{Results on the gridworld for different values of $q$ for $\| \nabla_\vtheta \log \pi(a|s)\|_q$, using 50 trajectories (10 runs, mean $\pm$ std).}
  \label{fig:qnorm}
\end{figure}
In the computation of our gradient-aware weights for all the experiments, we use $\| \nabla_\vtheta \log \pi(a|s)\|_q$ with $q=2$.
We validated our choice by running 15 iterations of GAMPS using 50 randomly collected trajectories and $q \in \left\{ 1, 2, \infty \right\}$.
As shown in Figure~\ref{fig:qnorm}, we did not find the choice of $q$ to be crucial for the performance of our algorithm.

\subsection{Minigolf} \label{app:minigolf}
We adopt the version of this domain proposed by \citep{tirinzoni2019transfer}. In the following, we report a brief description of the problem.

In the minigolf game, the agent has to shoot a ball with radius $r$ inside a hole of diameter $D$ with the minimum number of strokes. We assume that the ball moves along a level surface with a constant deceleration $d=\frac{5}{7}\rho g$, where $\rho$ is the dynamic friction coefficient between the ball and the ground and $g$ is the gravitational acceleration.
Given the distance $x_0$ of the ball from the hole, the agent must determine the angular velocity $\omega$ of the putter that determines the initial velocity $v_0 = \omega l$ (where $l$ is the length of the putter) to put the ball in the hole in one strike. For each distance $x_0$, the ball falls in the hole if its initial velocity $v_0$ ranges from $v_{min} = \sqrt{2dx_0}$ to $v_{max}=\sqrt{(2D-r)^2\frac{g}{2r}+v_{min}^2}$. $v_{max}$ is the maximum allowed speed of the edge of the hole to let the ball enter the hole and not to overcome it. 
At the beginning of each trial the ball is placed at random, between $2000$ cm and $0$ cm far from the hole. 
At each step, the agent chooses an action that determines the initial velocity $v_0$ of the ball. When the ball enters the hole the episode ends with reward $0$. If $v_0 > v_{max}$ the ball is lost and the episode ends with reward $-100$. Finally, if $v_0 < v_{min}$ the episode goes on and the agent can try another hit with reward $-1$ from position $x = x_0 - \frac{(v_0)^2}{2d}$. 
The angular speed of the putter is determined by the action $a$ selected by the agent as follows: $\omega = a l(1 + \epsilon)$, where $\epsilon\sim\mathcal{N}(0,0.3)$.
This implies that the stronger the action chosen the more uncertain its outcome will be. As a result, the agent is disencumbered by trying to make a hole in one shot when it is away from the hole and will prefer to perform a sequence of approach shots.
The state space is divided into two parts: the first one, bigger twice the other, is the nearest to the hole and features $\rho_1=0.131$; the second one is smaller and has an higher friction with $\rho_1=0.19$.

We use a linear-Gaussian policy that is linear on six equally-spaced radial basis function features. Four of the basis functions are therefore in the first area, while two are in the other one. The parameters of the policy are initialized equal to one for the mean and equal to zero for the standard deviation.

As a model class, we use parameterized linear-Gaussian models that predict the next state by difference with respect to the previous one. We avoid the predictions of states that are to the right with respect to the current state by using a rectifier function. The overall prediction of the next state by the model is given by $\widehat{s}_{t+1} = s_t - \max(0, \epsilon), \epsilon \sim \mathcal{N}(V_\mu[s_t,a_t],V_\sigma[s_t,a_t])$, where $V_\mu$ and $V_\sigma$ are two learnable parameters.

For all the learning algorithms, we employ a constant learning rate of $0.08$ for the Adam optimizer, with $\beta_1=0$ and $\beta_2=0.999$. For training the model used by GAMPS and ML, we minimize the MSE weighted through our weighting scheme, again using Adam with learning rate $0.02$ and default betas. For the estimation of the Q-function, we use the on-the-fly procedure outlined in Section \ref{sec:ComputingQ}, with an horizon of 20 and averaging over 10 rollouts. Also in this experiment we set $q=2$ for the q-norm $\| \nabla_\vtheta \log \pi(a|s)\|_q$ of the score. We use $\gamma=0.99$.

\subsection{Swimmer}
In the swimmer task, a 3-link swimming robot is able to move inside a viscous fluid. The goal is to control the two joints of the robot in order to make it swim forward.
The fully-observable state space consists of various positions and velocities, for a total of 8 scalars.
The reward function is a linear combination of a \textit{forward term}, determined by how much the action of the agent allowed it to move forward, and a \textit{control term}, consisting of the norm of its actions.
The two rewards are combined by means of a control coefficient $\alpha_{\mathrm{ctrl}}=0.0001$.

For running GAMPS on this task, we chose to make use of more powerful model classes, in order to show that gradient-aware model learning can be effective even in the case of high-capacity regimes.
Thus, we use 2-layer neural networks with 32 hidden units, that take current states and actions as inputs.
To better model the uncertainty about the true model, we output a parameterized mean and diagonal covariance.
Then, we sample from the resulting normal distribution the difference from the previous state.
At each iteration, we train the model for 300 epochs with a learning rate of 0.0002 using Adam with default $\beta_1$ and $\beta_2$.
We found beneficial, to reduce the computation burden, to employ a stopping condition on the learning of the model, stopping the training if no improvement in the training loss is detected for 5 epochs.

For computing the approximate $Q^{\pi,\widehat{p}}$, we average the cumulative return obtained by rolling our the model for 20 rollouts composed of 25 steps.
The policy is learned using Adam with a learning rate of 0.008 and default $\beta$s.
We use a discount factor of 0.99 and $q=2$ in the computation of the weighted MSE loss used in model learning.

\section{Algorithm}
In this appendix, we report additional details on the GAMPS algorithm.

\subsection{Time complexity of Algorithm \ref{alg:GAMPS}}
Let us consider that the algorithm is run for $K$ iterations on a dataset of $N$ trajectories.
 Suppose a parametric model class for which at most $E$ epochs are necessary for estimation. We define $H$ as the maximum length of a trajectory (or \emph{horizon}) and use an estimate of the Q-function derived by sampling $M$ trajectories from the estimated model, as described in Section~\ref{sec:ComputingQ}. 
For every iteration, we first compute the weights for every transition in every trajectory $\mathcal{O}(NH)$ and then estimate the corresponding forward model (order of $NHE$). 
Then, we estimate the gradient given all the transitions, using the trajectories imagined by the model for obtaining the value function (order of $NMH^2$). 
The overall time complexity of the algorithm is therefore $\mathcal{O}(KNHE + KNMH^2)$.
%

\subsection{Approximation of the value function} \label{sec:from_p_to_q}
We now briefly review in a formal way how $Q^{\pi,\widehat{p}}$ can be estimated. 
For the discrete case, the standard solution is to find the fixed point of the Bellman equation:
\begin{align}
	\label{eq:dynamic_programming_q}
    \widehat{Q}(s,a) &= r(s,a) + \gamma \mathop{\mathbb{E}}_{\substack{s' \sim \widehat{p}(\cdot|s,a) \\ a' \sim \pi(\cdot|s')}} \left[ \widehat{Q}(s',a') \right],
\end{align}
that can be found either in exact form using matrix inversion or by applying Dynamic Programming.
In the continuous case, one can use approximate dynamic programming. For instance, with one step of model unrolling, the state-action value function could be found by iteratively solving the following optimization problem:
\begin{equation}
	\label{eq:approximate_q}
    \widehat{Q} \in \argmin_{Q \in \mathcal{Q}} \sum_{i=1}^N \sum_{t=1}^{T_i-1} \left( Q(s_t^i,a_t^i) - \left(r(s_t^i,a_t^i) + \gamma \mathop{\mathbb{E}}_{\substack{\overline{s}_{t+1}^i \sim \widehat{p}(\cdot|s_t^i,a_t^i) \\ \overline{a}_{t+1}^i \sim \pi(\cdot|\overline{s}_{t+1}^i)}} \left[ Q(\overline{s}_{t+1}^i,\overline{a}_{t+1}^i) \right] \right)\right)^2. 
\end{equation}
The expected value in Equation \eqref{eq:approximate_q} can be approximated by sampling from the estimated model $\widehat{p}$ and the policy $\pi$. In practice, a further parameterized state-value function $\widehat{V}(s) \approx \mathbb{E}_{a \sim \pi(\cdot|s)} \left[ \widehat{Q}(s,a) \right]$ can be learned jointly with the action-value function. 

The third approach, that is the one employed in GAMPS, is to directly use the estimated model for computing the expected cumulative return starting from $(s,a)$. We can therefore use an \textit{ephemeral} Q-function, that is obtained by unrolling the estimated model and computing the reward using the known reward function. 

\section{A connection with reward-weighted regression}
Interestingly, our gradient-aware procedure for model learning has some connections with the reward-weighted regression (RWR,~\citealp{peters2007reinforcement}) techniques, that solve reinforcement learning problems by optimizing a supervised loss.
To see this, we shall totally revert our perspective on a non-Markovian decision process. 
First, we interpret a model $\widehat{p}_\phi$ parameterized by $\phi$ as a \emph{policy}, whose action is to pick a new state after observing a previous state-action combination. 
Then, we see the policy $\pi$ as the \emph{model}, that samples the transition to the next state given the output of $\widehat{p}_\phi$. Finally, the cumulative absolute score at time $t$ is the (non-markovian) reward. 
To strengthen the parallel, let us consider an appropriate transformation $u_c$ on the weights $\omega_t$.

We can now give an expectation-maximization formulation for our model learning problem as reward-weighted regression in this newly defined decision process: \begin{itemize}
\item \textit{E-step}:
\begin{equation}
    q_{k+1}(t)=\frac{p_{\phi_k}(s_{t+1}|s_t,a_t) u_{c_k}(\omega_{t})}{\sum_{t'} p_{\phi_k}(s_{t'+1}|s_{t'},a_{t'}) u_{c_k}(\omega_{t'})}
\end{equation}
\item \textit{M-step for model parameters}:
\begin{equation}
    \phi_{k+1}= \argmax \sum_t q_{k+1}(t) \log p_\phi(s_{t+1}|s_t,a_t)
\end{equation}
\item \textit{M-step for transformation coefficient}:
\begin{equation}
    \tau_{k+1} = \argmax_c \sum_t q_{k+1}(t) u_c(\omega_t)  
\end{equation}
\end{itemize}

Assuming a Gaussian-linear model $\widehat{p} = \mathcal{N}(s_{t+1}|\mu(s_t,a_t),\sigma^2\mathbf{I})$ and a transformation $u_c(x) = c \exp(-c x)$, the update for the model parameters and the transformation parameter is given by:
\begin{align}
\phi_{k+1} = \mathbf{(\Phi^TW\Phi)^{-1}\Phi^TWY} \\
\sigma^2_{k+1}=\mathbf{\|Y - \phi_{k+1}^T\Phi\|^2_W} \\
c_{k+1} = \frac{\sum_t u_c(\omega_t)}{\sum_{t'} u_c(\omega_{t'})\omega_{t'}}
\end{align}
where $\mathbf{\Phi, Y \text{ and } W}$ are the matrices containing, respectively, state-action features, successor state features and cumulative score weights on the diagonal.

As in the case of the original RWR, this learned exponentiation of the weights could in practice improve the performance of our algorithm. We leave this direction to future work.

\end{document}